\documentclass{article}

\usepackage{amsmath, mathrsfs, amssymb, amsthm, mathtools, bbm} 

\usepackage{caption}
\usepackage{subcaption}

\usepackage{makecell} 
\usepackage{float} 

\usepackage[english]{babel} 
\usepackage[utf8]{inputenc} 
\usepackage[T1]{fontenc} 
\usepackage[margin=1.325in,footskip=0.35in]{geometry} 

\usepackage{tikz} 
\usetikzlibrary[topaths] 

\usepackage[symbol]{footmisc}

\usepackage{enumitem}

\usepackage{algorithm}
\usepackage{algcompatible} 

\newtheorem{theorem}{Theorem}
 
\newtheorem{definition}{Definition}
\newtheorem{lemma}{Lemma}
\newtheorem{remark}{Remark} 
\newtheorem{corollary}{Corollary}
\newtheorem{proposition}{Proposition}

\newtheorem*{theorem*}{Theorem}
\newtheorem*{example*}{Example} 
\newtheorem*{definition*}{Definition}
\newtheorem*{lemma*}{Lemma}
\newtheorem*{remark*}{Remark}
\newtheorem*{corollary*}{Corollary}
\newtheorem*{proposition*}{Proposition}
\newtheorem*{assumption*}{Assumption}
\newtheorem*{claim*}{Claim}

\newtheoremstyle{TheoremNum}
        {\topsep}{\topsep}              
        {\itshape}                      
        {}                              
        {\bfseries}                     
        {.}                             
        { }                             
        {\thmname{#1}\thmnote{ \bfseries #3}}
\theoremstyle{TheoremNum}

\newtheoremstyle{LemmaNum}
        {\topsep}{\topsep}              
        {\itshape}                      
        {}                              
        {\bfseries}                     
        {.}                             
        { }                             
        {\thmname{#1}\thmnote{ \bfseries #3}}
\theoremstyle{LemmaNum}

\usepackage{hyperref}

\newcommand{\Ind}{ \mathbb{I} }

\newcommand{\grad}{\mathrm{grad}}

\newcommand{\Proj}{\mathrm{Proj}}

\newcommand{\M}{\mathcal{M}}
\newcommand{\D}{\mathcal{D}}
\newcommand{\Exp}{{\mathrm{Exp}}}

\newcommand{\R}{\mathbb{R}}

\renewcommand{\S}{\mathbb{S}}
\newcommand{\E}{\mathbb{E}}

\renewcommand{\P}{\mathcal{P}}

\renewcommand{\[}{\left[ }
\renewcommand{\]}{\right] }

\newcommand{\<}{\left< }
\renewcommand{\>}{\right> }

\renewcommand{\(}{\left( }
\renewcommand{\)}{\right) }

\newcommand{\wt}{\widetilde }
\newcommand{\wh}{\widehat } 

\usepackage{natbib}

\begin{document}

\title{From the Greene--Wu Convolution to Gradient Estimation over Riemannian Manifolds} 

\author{Tianyu Wang\footnote{wangtianyu@fudan.edu.cn} \quad Yifeng Huang\footnote{huangyf@umich.edu} \quad Didong Li\footnote{didongli@princeton.edu}} 

\date{}

\maketitle 

\begin{abstract}
    Over a complete Riemannian manifold of finite dimension, Greene and Wu studied a convolution of $f$ defined as 
    \begin{align*} 
        \widehat{f}^\mu (x) :=  \int_{ v \in T_p \M  } f \left( \Exp_p (v ) \right) \kappa_\mu \left( v \right) d \Pi , 
    \end{align*} 
    where $\kappa_\mu $ is a kernel that integrates to 1, and $d\Pi$ is the base measure on $T_p \mathcal{M} $. 
    In this paper, we study properties of the Greene--Wu (GW) convolution and apply it to non-Euclidean machine learning problems. 
    In particular, we derive a new formula for how the curvature of the space would affect the curvature of the function through the GW convolution. 
    Also, following the study of the GW convolution, a new method for gradient estimation over Riemannian manifolds is introduced. 
\end{abstract}

\section{Introduction}




Recently, 
as data structures are becoming increasingly non-Euclidean, many non-Euclidean operations are studied and applied to machine learning problems \citep[e.g.,][]{absil2009optimization}. Among these operations, the Greene-Wu (GW) convolution \citep{greene1973subharmonicity,greene1976c,greene1979c} is important but relatively less understood.

Over a complete Riemannian manifold, the seminal Greene-Wu convolution (or approximation) of function $f$ at $p$ is defined as 
\begin{align}
    \wh{f}^\mu (p) :=  \int_{ v \in T_p \M  } f \left( \Exp_p (v ) \right) \kappa_\mu \left( v \right) d \Pi , \tag{GW}
\end{align}
where $\kappa_\mu $ is a convolution kernel that integrates to 1, and $d\Pi$ is the base measure on $T_p \M $. 
The GW convolution generalizes standard convolutions in Euclidean spaces and has been subsequently studied in mathematics \citep[e.g.,][]{parkkonen2012strictly,azagra2006inf,azagra2013global}. 
This convolution operation also naturally arises in machine learning scenarios. For example, if one is interested in studying a function over a matrix manifold, the output of such a function $f$ is GW convoluted if the input is noisy (i.e., we can only obtain an average function value near $p$ whenever we try to obtain the function value at $p$). A real-life scenario where the GW convolution happens is photo taking with shaky camera. Any function defined over the manifolds of images (e.g., likelihood of containing a cat) is GW convoluted if the camera is perturbed. 

Despite the importance of the GW convolution, many of its elementary properties are only inadequately understood. In particular, \textit{how would the curvature of the manifold affect the curvature of the function through the GW convolution?} (\textbf{Q1}) The answer to this question is important to machine learning problems, since the curvature of a function depicts fundamental properties of a function, including how convex the function is. In the above-mentioned example scenario, a better understanding of the GW convolution would lead to a better understanding of how the landscape of the function $f$ would be affected by the curvature of the space when the input is noisy. 

In this paper, we give a quantitative answer to question \textbf{(Q1)}. 
As an example, consider $ \wh{f}^\mu $ the GW convolution defined with kernel $\kappa_\mu$ whose radial density (see Definition \ref{def:kernel}) is uniform over $[-\mu, \mu]$. For this GW convolution, we show that 
\begin{align*}
    \lambda_{\min} \( H_{p, \wh{f}^\mu} \) 
    \ge& \; 
    \frac{1}{2\mu} \E_{v \sim \S_p } \hspace*{-2pt} \[ \int_{-\mu}^\mu \lambda_{\min} \( H_{\Exp_p ( tv ), f } \) dt \] \\ 
    &+ \; 
    \min_{u \in \S_p } \frac{1}{2 \mu } \E_{v \sim \S_p} \[ \int_{-\mu}^\mu  \sum_{j = 1}^\infty \frac{ t^{2j} }{ (2j)! } \nabla_{ u}^2 \nabla_{v}^{2j} f (p) \, dt - \int_{-\mu}^\mu \sum_{j = 1}^\infty \frac{ t^{2j} }{ (2j)!} \nabla_{v}^{2j} \nabla_{ u}^2 f (p) \, dt \] , 
\end{align*}
where 
\begin{enumerate} 
    \item for any twice continuously differentiable $f$ and any $p \in \M$, $H_{p, {f}}$ denotes the Hessian matrix of $f$ at $p$; 
    \item $ \lambda_{\min}  $ exacts the minimal eigenvalue of a matrix; 
    \item $\S_p$ denotes the unit sphere in $T_p \M$ and $ \E_{v \sim \S_p} $ denotes the expectation taken with respect to $v$ uniformly sampled from $\S_p$; 
    \item $\nabla$ denotes the covariant derivative associated with the Levi-Civita connection. 
\end{enumerate}

This result implies that the GW convolution can sometimes make the function more convex, and thus often more friendly to optimize. Specifically, if $f$ and $\nabla$ satisfy that 
$$
\min_{u \in \S_p } \frac{1}{2 \mu } \E_{v \sim \S_p} \[ \int_{-\mu}^\mu  \sum_{j = 1}^\infty \frac{ t^{2j} }{ (2j)! } \nabla_{ u}^2 \nabla_{v}^{2j} f (p) \, dt - \int_{-\mu}^\mu \sum_{j = 1}^\infty \frac{ t^{2j} }{ (2j)!} \nabla_{v}^{2j} \nabla_{ u}^2 f (p) \, dt \] > 0, 
$$ 
then the minimal eigenvalue of $\wt{f}^\mu$ at $p$ is larger than the minimal eigenvalue of $f$ near $p$. This means sometimes noisy input can make the function more convex and thus more friendly to optimize, thanks to the curvature of the space. 

Besides understanding of how the curvature of the space would affect curvature of the function via GW convolution, another important question is \textit{how can we directly apply the GW convolution operation to machine learning problems?}  \textbf{(Q2)} Naturally, the GW convolution is related to anti-derivatives and thus gradient estimation, as suggested by the fundamental theorem of calculus or the Stokes' theorem. Along this line, we introduce a new gradient estimation method over Riemannian manifolds. Our gradient estimation method improves the state-of-the-art scaling with dimension $n$ from $  \( n+3 \)^{3/2} $ to $ n^{3/2} $, while holding all other quantities the same\footnote{More details for the comparison are explained in Remark \ref{remark:compare}.}. Apart from curvature, Riemannian gradient estimation differs from its Euclidean counterpart in that an open set is diffeomorphic to an Euclidean set only within the injectivity radius. This difference implies that one needs to be extra careful when taking finite difference steps in non-Euclidean spaces. To this end, we show that as long as the function is geodesically $L_1$-smooth, the finite difference method always works, no matter how small the injectivity radius is. 
Empirically, our method outperforms the best existing method for gradient estimation over Riemannian manifolds, as evidenced by thorough experimental evaluations. Based on our answer to \textbf{(Q2)}, we study the online convex learning problems over Hadamard manifolds under stochastic bandit feedback. In particular, single point gradient estimator is designed. Thus online convex learning problem  under bandit-type feedback can be solved, and a curvature-dependent regret bound is derived.  


\subsection*{Related Works} 

The study of geometric methods in machine learning has never stalled. Researchers have investigated this topic from different angles, include kernel methods \citep[e.g.,][]{scholkopf2002learning,muandet2017kernel,jacot2021neural}, manifold learning \citep[e.g.,][]{lin2008riemannian,li2017efficient,zhu2018image,li2020classification}, metric learning \citep[e.g.,][]{kulis2012metric,li2019geodesic}. Recently, attentions are concentrated on problems of learning with Riemannian geometry. Examples include convex optimization method over negatively curved spaces \citep{zhang2016first,zhang2016riemannian}, online learning algorithms with Riemannian metric \citep{antonakopoulos2019online}, mirror map based methods for relative Lipschitz objectives \citep{zhou2020regret}.

Despite the rich literature on geometric methods for machine learning \citep[e.g.,][]{absil2009optimization}, some important non-Euclidean operations are relatively poorly understood, including the Greene--Wu (GW) convolution \citep{greene1973subharmonicity,greene1976c,greene1979c,parkkonen2012strictly,azagra2006inf,azagra2013global}. To this end, we derive new properties of the GW convolution and apply it to machine learning problems. 

Enabled by our study of the GW convolution, we design a new gradient estimation method over Riemannian manifolds. Perhaps the history of gradient estimation can date back to Sir Isaac Newton and the fundamental theory of calculus. Since then, many gradient estimation methods have been invented  \citep{lyness1967numerical,hildebrand1987introduction}. In its more recent form, gradient estimation methods are often combined with learning and optimization. \citet{flaxman2005online} used the fundamental theorem of geometric calculus and solved online convex optimization problems with bandit feedback. \citet{nesterov2017random,li2020stochastic} introduced a stochastic gradient estimation method using Gaussian sampling, in Euclidean space or manifolds embedded in Euclidean spaces. 
We improve the previous results in two ways. 1. Our gradient estimation method generalizes this previous work by removing the ambient Euclidean space. 2 We tighten the previous error bound. Over an $n$-dimensional Riemannian manifold, we show that the error of gradient estimation is bounded by $ \frac{L_1 \( n+3 \)^{3/2} \mu }{2} $ in contrast to $ \frac{L_1 n^{3/2} \mu }{2} $\citep{li2020stochastic}, where $n$ is the dimension of the manifold, $ L_1 $ is the smoothness parameter of the objective function and $\mu$ is an algorithm parameter controlling finite-difference step size.

The application of our method to convex bandit optimization is related to online convex learning \citep[e.g.,][]{shalev2011online}. The problem of online convex optimization with bandit feedback was introduced and studied by \cite{flaxman2005online}. Since then, a sequence of works has focused on bandit convex learning in Euclidean spaces. For example, \cite{hazan2012projection, chen2019projection, garber2020improved} have studied the gradient-free counterpart for the Frank-Wolfe algorithm, and \cite{hazan2014bandit} studied this problem with self-concordant barrier. An extension to this scheme is combining the methodology with a partitioning of the space \cite{kleinberg2004nearly}. By doing so, regret rate of $\mathcal{O} \( \sqrt{T} \)$  can be achieved in Euclidean spaces \citep{bubeck2015bandit,bubeck2016multi}. This rate is optimal in terms of $ T $ for stochastic environement, as shown by \cite{shamir2013complexity}. 
While there has been extensive study of online convex optimization problems, they are all restricted to Euclidean spaces. In this paper, we generalize previous results from Euclidean space to Hadamard manifolds, using our new gradient estimation method.

\section{Preliminaries and Conventions}

\subsection*{Preliminaries for Riemannian Geometry} 

A Riemannian manifold $ (\M, g) $ is a smooth manifold equipped with a metric tensor $g$ such that at any point $p \in \M$, $g_p$ defines a positive definite inner product in the tangent space $T_p\M$. For a point $p \in \M$, we use $ \< , \>_p $ to denote the inner product in $T_p \M$ induced by $g$. Also, we use $\|\cdot \|_p $ to denote the norm associated with $ \< ,\>_p $. Intuitively, a Riemannian manifold is a space that can be locally identified as an Euclidean space. With such structures, one can locally carry out calculus and linear algebra operations.  

A vector field $X$ assigns each point $p \in \M$ a unique vector in $T_p \M$. Let $\mathfrak{X}(\M)$ be the space of all smooth vector fields. An affine connection is a mapping 
$\nabla: \mathfrak{X}(\M)\times \mathfrak{X}(\M) \rightarrow \mathfrak{X}(\M)$  so that (i) $ \nabla_{f_1 X_1 + f_2 X_2 } Y = f \nabla_{X_1} Y  + f_2 \nabla_{ X_2 } Y$ for $f_1, f_2 \in C^\infty (\M)$, (ii) $ \nabla_{X } (a_1 Y_1 + a_2 Y_2) = a_1 \nabla_{X} Y_1 + a_2 \nabla_X Y_2$ for $a_1, a_2 \in \R $, and (iii) $ \nabla_{X } (f Y) = f \nabla_{X } Y + X(f) Y $ for all $f \in C^\infty (\M)$. The Levi-Civita connection is the unique affine connection that is torsion free and 
compatible with the Riemannian metric $g$. 
A geodesic curve $\gamma: [0,1] \rightarrow \M$ is a curve so that $\nabla_{\dot{\gamma}(t)} \dot{\gamma}(t) = 0$ for all $t \in [0,1]$, where $ \nabla $ here denotes the Levi-Civita connection. 
One can think of geodesic curves as generalizations of straight line segments in Euclidean space. At any point $p \in \M$, the exponential map $\Exp_p$ is a local diffeomorphism that maps $v \in T_p\M$ to $\Exp_p (v) \in \M$ such that there exists a geodesic curve $\gamma : [0,1] \rightarrow \M$ with $ \gamma (0) = p $, $ \dot{\gamma}\big|_0 = v $ and $\gamma(1)=\Exp_p (v)$. Since the exponential map is (locally) diffeomorphic, the inverse exponential map is also properly defined (locally). We use $\Exp_p^{-1}$ to denote the inverse exponential map of $\Exp_p$. 


We say a vector field $X$ is parallel along a geodesic $\gamma: [0,1] \rightarrow \M$ if 
$ \nabla_{\dot{\gamma} (t)} X = 0 $ for all $t\in[0,1]$. For $t_0, t_1 \in [0,1]$, we call the map $ \P_{t_0, t_1}^{\gamma} : T_{\gamma (t_0)} \M \rightarrow T_{\gamma (t_1)} \M $ a parallel transport, if it holds that for any vector field $X$ such that $X$ is parallel along $\gamma$, $X (\gamma (t_0)) $ is mapped to $X (\gamma (t_1)) $. Parallel transport induced by the Levi-Civita connection preserves the Riemannian metric, that is, for $p\in \M$, $u,v \in T_p \M$, and geodesic $\gamma : [0,1] \rightarrow \M$, it holds that $\< u,v \>_{\gamma(0)} = \< \P_{0,t}^\gamma (u), \P_{0,t}^\gamma (v) \>_{\gamma(t)}$ for any $t \in [0,1]$. 

The Riemann curvature tensor $\mathcal{R}:\mathfrak{X}(\M)\times\mathfrak{X}(\M)\times \mathfrak{X}(\M)\to\mathfrak{X}(\M)$ is defined as $\mathcal{R}(X,Y)Z = \nabla_X\nabla_YZ-\nabla_Y\nabla_X Z-\nabla_{[X,Y]}Z$, where $[,]$ is the Lie bracket for vector fields. The sectional curvature $\mathcal{K}:\mathfrak{X}(\M)\times\mathfrak{X}(\M)\to\mathbb{R}$ is defined as $\mathcal{K}(X,Y)=\frac{g(\mathcal{R}(X,Y)X,Y)}{g(X,X)g(Y,Y)-g(X,Y)^2}$, which is the high dimensional generalization of Gaussian curvature. Intuitively, for a point $p$, $\mathcal{K} \( X(p), Y (p) \)$ is the the Gaussian curvature of the sub-manifold spanned by $  X(p), Y(p)$ at $p$. 

For a continuously connected Riemannian manifold $(\M, g)$, a distance metric can be defined. The distance between $p,q \in \M$ is the infimum over the lengths of piecewise smooth curves in $ \M $ connecting $p$ and $q$, where the length of a curve is defined with respect to the metric tensor $g$.  A Riemannian manifold is complete if it is complete as a metric space. By the celebrated Hopf–Rinow theorem, if a Riemannian manifold is complete, the exponential map $\Exp_p$ is defined over the entire tangent space $T_p \M$ for all $p \in \M$. The cut locus of a point $p \in \M$ is the set of points $q \in \M$ such that there exists more than two distinct shortest path from $ p $ to $q$. 

We refer the readers to books \citep[e.g.,][]{petersen2006riemannian,lee2006riemannian} for more background on Riemannian geometry. 

\subsection*{Notations and Conventions}
For better readability, we list here some notations and conventions that will be used throughout the rest of this paper. 

\begin{itemize}[align=left,style=nextline,leftmargin=*]
    \item For any $ x \in \M$, let $ U_p $ denote the open set near $p$ that is a diffeomorphic to a subset of $\R^n$ via the local normal coordinate chart $\phi$. Define the distance $d_p (q_1, q_2)$ ($q_1,q_2 \in U_p$) such that  
    \begin{align*} 
        d_p (q_1, q_2) = d_{\text{Euc}} ( \phi (q_1) , \phi (q_2) ). 
    \end{align*} 
    where $d_{\text{Euc}}$ is the Euclidean distance in $\R^n$. 
    
    \item As mentioned in the preliminaries, for any $p \in \M$, we use $ \< \cdot, \cdot\>_p $ and $\| \cdot \|_p$ to denote the inner produce and norm induced by the Riemannian metric $g$ in the tangent space $T_p \M$. We omit the subscript when it is clear from context. 
    
    \item For any $p \in \M$ and differentiable function $f $ over $\M$, We use $ \grad f \big|_p $ to denote the gradient of $f$ at $p$.
    
    \item For any $p \in \M$ and $\alpha > 0$, we use $ \S_p (\alpha) $ to denote the origin-centered sphere in $T_p \M$ with radius $\alpha$. For simplicity, we  write $ \S_p = \S_p (1) $. 
    
    \item We will use another set of notations for parallel transport. For two points $p,q \in \M$ and a smooth curve $\gamma$ connecting $p$ and $q$, we use $\P_{p \rightarrow q}^\gamma$ to denote the parallel transport from $T_p \M$ to $T_q \M$ along the curve $\gamma$. Let $ \Omega_{\min} (p,q) $ denote the set of minimizing geodesics connecting $p$ and $q$. Define $ \P_{p \rightarrow q} : T_p \M \rightarrow T_q \M $ be defined such that $ \P_{p \rightarrow q} (v) = \E_{\gamma \sim \Omega_{\min} (p,q) } \[ \P_{p \rightarrow q}^\gamma (v) \] $ for any $v \in T_p \M$, where the expectation is taken with respect to the uniform probability measure over $ \Omega_{\min} (p,q) $. Since for any point $p$, its cut locus is of measure zero, the operation $ \P_{p \rightarrow q} $ is deterministically defined for almost every $q$. 
    
    \item For any $p \in \M$, $v \in T_p \M$ and $q \in \M$, define $v_q = \P_{p \rightarrow q} (v)$. 
    
   
    \item As a convention, if $U \subseteq \M $ is diffeomorphic to $ V \subseteq \R^n $, we simply say $ U $ is diffeomorphic. Similarly, a map is said to be diffeomorphic when the domain of the map is diffeomorphic to an open subset of the Euclidean space. 
    For simplicity, we omit all musical isomorphisms when there is no confusion. 
\end{itemize} 

\section{The Greene--Wu Convolution} 

The Greene--Wu convolution can be decomposed into two steps. In the first step, a direction $ v $ in the tangent space is picked, a convolution along the direction is performed, and the extension of this convolution to other points is defined. In the second step, the direction $v$ is randomized over a sphere. To properly describe these two steps, we need to formally define kernels. 

\begin{definition}
    \label{def:kernel}
    A function $\kappa_\mu: \R^n \rightarrow \R_{\ge0}$ is called a kernel if 
    \begin{enumerate}
        \item For any $\alpha>0$ and $\mu > 0$, there exists a density $\Theta_{\mu, \alpha} $ on $ \R $ such that $ \kappa_\mu(v) = {2} \nu_\alpha \( \alpha \frac{v}{\| v \| } \) \Theta_{\mu,\alpha} \( \frac{ \| v \| }{ \alpha } \) $, and $ \Theta_{\mu,\alpha} (-x) = \Theta_{\mu,\alpha} (x) $ for all $x \in \R$, where $ \nu_\alpha $ is the uniform probability density over the Euclidean sphere of radius $\alpha$\footnote{Here $\kappa_\mu(0) = 2 \Theta_{\mu,\alpha} (0)$ as a convention.}; 
        \item $ \lim_{\mu \rightarrow 0^+} \kappa_\mu= \delta_0 $,
        where $\delta_0$ is the Dirac point mass at $0$. 
    \end{enumerate} 
    We call $ \Theta_{\mu, \alpha} $ the radial density for $ k_\mu $. 
\end{definition}

By the Radon-Nikodym theorem, item 1 in the definition of a kernel is satisfied when the measure on $T_p \M$ induced by $\kappa_\mu$ is absolutely continuous with respect to the base measure on $T_p \M$. Examples of kernels include uniform mass over a ball of radius $\mu$. 

\begin{remark} 
    \label{remark:general} 
    For simplicity, we restrict our attention to symmetric kernels $\kappa_\mu$, where the value of $\kappa_\mu (v)$ only depends on $ \| v \| $. Yet many of the results generalize to non-symmetric kernels. To see this, we 
    decompose $\kappa_\mu$ by $ \kappa_{\mu} (v) = 2 \nu \( \alpha \frac{v}{\| v\|} \) \Theta_{\mu, \alpha ; v} \( \frac{ \| v \| }{\alpha} \) $, where $ \nu_\alpha $ is the uniform density over the sphere of radius $ \alpha $, and $ \Theta_{\mu, \alpha ; v} $ is the density induced by $\kappa_\mu$ along the direction of $v$ (or a parallel transport of $v$) up to a scaling by $\alpha$. We can replace item 1 in Definition \ref{def:kernel} and many of the follow-up results can be extended to these cases. 
\end{remark}

With the kernel defined in Definition \ref{def:kernel}, the GW approximation/convolution can be defined as follows. For $p \in \M$ and $\mu >0$, we define 
\begin{align} 
    \widetilde{f}^\mu (p) := \E_{ v \sim \S_p (\alpha) } \[ \wt{f}_{v}^\mu (p) \],  \label{eq:def-surrogate} 
\end{align} 
where 
\begin{align} 
    \wt{f}_{v}^\mu(q)= \int_{ \R } 2 f \( \Exp_q ( t v_{ q } ) \) \Theta_{\mu,\alpha} (t ) dt, \quad \forall q \in \M.  \label{eq:fvd} 
\end{align} 

Next we show that the convolution in (\ref{eq:def-surrogate}) is identical to the (GW) convolution everywhere.

\begin{proposition}
    \label{prop:equiv} 
	For any $p \in \M$ and $\kappa_\mu$ defined on $T_p \M \cong \R^n$, it holds that $ \wt{f}^\mu (p) = \wh{f}^\mu (p) $. 
\end{proposition} 

\begin{proof}
    Let $dt$ be the base measure on $\R$ and let $ d S_\alpha $ be the base measure on $\S_p (\alpha)$. We then have $ d t d S_\alpha = d \Pi $. 
    Thus it holds that
    \begin{align*}
        \wt{f}^\mu (p) =& \int_{v \in \S_p (\alpha)} \int_{t \in \R } 2 f \( \Exp_p ( t v ) \) \nu_\alpha \( \alpha \frac{ v }{ \| v \| } \) \Theta_{\mu,\alpha} ( t  ) \, dt d S_\alpha \\
        =& \int_{u \in T_p \M } f \( \Exp_p (u) \) \kappa_\mu(u) \, d \Pi = \wh{f}^\mu (p), 
    \end{align*} 
    where a change of variable $u = tv$ is used. 
    
\end{proof} 


What's interesting with (\ref{eq:def-surrogate}) is that it extends (GW) from a local definition to a global definition, as pointed out in Remark \ref{remark}. 

\begin{remark} 
    \label{remark} 
    With the Levi-Civita connection, spheres are preserved after parallel transport, and so are the uniform base measures over the spheres. Thus the spherical integration of $\wt{f}_v^\mu$ defined in (\ref{eq:fvd}) can be carried out at $p$ regardless whether $v \in T_p \M$ or not. 
    Hence, the convolution defined by (\ref{eq:def-surrogate}) and (\ref{eq:fvd}) is indeed the GW convolution in the sense that $ \wt{f}_v^\mu $ agrees with (GW) almost everywhere\footnote{The cut locus is of measure zero and thus $v_q$ is deterministically defined almost everywhere.} if the convolution at every point is restricted to the direction specified by $v$ (or parallel transport of $v$). Now the advantage of our decomposition is clear -- $ \wt{f}_v^\mu $ is globally defined and agrees with (GW) along the the direction of $v$. This advantage leads to easy computations and meaningful applications, as we will see throughout the rest of the paper. 
\end{remark} 





\subsection{Properties of the GW Convolution} 

In their seminal works, Greene and Wu showed that the approximation is convexity-preserving when the width of the kernel approaches 0 (Theorem 2 by \citet{greene1973subharmonicity}). We show that the GW convolution is convexity-preserving even if the width of the kernel is strictly positive. Also, second order properties of the GW approximation is proved. 

\subsubsection{Convexity-Preserving Property of the GW Convolution}

\begin{definition}[Convexity]
	\label{def:geo-conv} 
    Let $\M$ be a complete Riemannian manifold. 
    A smooth function $f$ is convex near $ p$ with radius exceeding $\epsilon$ if 
    there exists $ \epsilon > 0 $, such that for any $ v \in \S_p $, the function $f \( \Exp_p (tv) \) $ is convex in $t$ over $[-\epsilon, \epsilon]$. 
\end{definition} 



Using our decomposition trick of the GW approximation (defined in Eq. \ref{eq:def-surrogate}), one can show that the GW convolution is convexity preserving even if the width of the kernel is strictly positive. 

\begin{theorem}
    \label{thm:approx}
	Let $\M$ be a complete Riemannian manifold, and fix an arbitrary $p \in \M $. If there exists $\epsilon > 0$ such that 1) $f$ is convex near $p$ with radius exceeding $\epsilon$, and 2) the density $\Theta_\mu$ (a shorthand for $ \Theta_{\mu, 1} $) induced by kernel $ \kappa_\mu$ satisfies that $ \int_{ |t| \ge \frac{\epsilon}{2} } \Theta_\mu (t) = 0 $, 
	then the Greene-Wu approximation $\wh{f}^\mu$ is also convex near $p$ with radius exceeding $ \frac{\epsilon}{2} $. 
\end{theorem} 

\begin{proof}[Proof of Theorem \ref{thm:approx}]
    
    By Proposition \ref{prop:equiv} and Remark \ref{remark}, it is sufficient to consider $ \wt{f}^\mu $. 
	Since $f$ is convex near $p$ with radius exceeding $\epsilon$, we have, for any $v \in T_p \M$ and any real numbers $\lambda, t \in \( 0, \frac{ \epsilon }{2} \)$,
	\begin{align} 
		f (\Exp_p ((\lambda + t) v )) \hspace{-2pt} + \hspace{-2pt} f (\Exp_p ((-\lambda + t) v )) \hspace{-2pt} \ge \hspace{-2pt} 2 f (\Exp_p (t v )). \label{eq:helper0}
	\end{align}
	
	Let $\gamma (s) = \Exp_p (sv)$ ($s \in [-\lambda-t, \lambda+ t]$) for simplicity. 
	We can then rewrite (\ref{eq:helper0}) as
	\begin{align} 
    	f ( \Exp_{ \gamma ( \lambda ) } ( t \P_{0,\lambda}^\gamma (v) ) )  +  f ( \Exp_{ \gamma ( -\lambda ) } ( t \P_{0,-\lambda}^\gamma (v) ) )  \ge& 2 f (\Exp_p (t v )) \label{eq:helper1}
	\end{align}
	We can integrate the first term in (\ref{eq:helper1}) with respect to $v$ and $ t $, and get 
	\begin{align}
    	&\int_{ v \in \S_p } \int_{t \in \[0, \frac{\epsilon}{2} \] } 2 f ( \Exp_{ \gamma ( \lambda ) } ( t \P_{0,\lambda}^\gamma (v) ) ) \nu_1 (v) \Theta_\mu (t) dt  d S_1 \nonumber \\
    	=& \int_{ v \in \S_{\gamma(\lambda)}  } \int_{ t \in \[0, \frac{\epsilon}{2} \] } 2 f ( \Exp_{ \gamma ( \lambda ) } ( t v ) ) \nu_1 (v) \Theta_\mu (t) dt \, d S_1,  \label{eq:helper2} 
	\end{align} 
	where (\ref{eq:helper2}) uses that parallel transport along geodesics preserves length and angle, for the change of tangent space for the outer integral.

	Since the density $ \Theta_\mu $ satisfies that $\int_{ |t| \ge \frac{\epsilon}{2}} \Theta_\mu (t) = 0$, the term on the right-hand-side of (\ref{eq:helper2}) is the exact definition of $\wt{f}^\mu (\gamma (\lambda))$. We repeat this computation for the second term in (\ref{eq:helper1}), and get 
	\begin{align}
    	\wt{f}^\mu (\gamma(\lambda) ) + \wt{f}^\mu (\gamma(-\lambda)) &\ge 2 \wt{f}^\mu (\gamma(0)), \quad \forall \lambda \in \[ 0,\frac{\epsilon}{2} \],  \label{eq:conv-for-limit}
	\end{align}
	which is the midpoint convexity (around $ \gamma (0) $, along $\gamma$) for $\wt{f}^\mu$. 
	Since midpoint convexity implies convexity for smooth functions, we know $\wt{f}^\mu$ is convex near $p$ along $\gamma$. Since the above argument is true for any $\gamma$, we know that $\wt{f}^\mu$ is convex near $p$ with radius exceeding $\frac{\epsilon}{2}$. 
	
\end{proof} 

\subsubsection*{Previous Convexity-Preserving Result}

Previously, \citet{greene1973subharmonicity, greene1976c, greene1979c} studied the GW convolution from an analytical and topological perspective, and showed that the GW convolution is convexity-preserving \citep{greene1973subharmonicity} when the width of the kernel approaches zero. For completeness, this previous result is summarized in Theorem \ref{thm:conv-old}. In Theorem \ref{thm:approx} above, we show that the output of GW convolution is convex with a strictly positive radius, as long as the input is convex with a strictly positive radius.  

\begin{theorem}[\citet{greene1973subharmonicity}]
    \label{thm:conv-old}    
    Consider a smooth function defined over a complete Riemannian manifold $\M$. Let $ \wh{f}^\mu $ be the GW approximation of $f$ with kernel $\kappa_\mu$. Then it holds that 
    \begin{align} 
        \underset{\mu \rightarrow 0^+}{\lim \inf} \( \inf_{C \in \mathcal{C}_p } \frac{d^2 \wh{f}^\mu (C(t))}{dt^2}  \Big|_{t=0} \) \ge 0,  
    \end{align}
    where $\mathcal{C}_p $ is the set of all geodesics $C$ parametrized by arc-length such that $C(0) = p$. 
\end{theorem} 

\subsection{Second Order Properties of the GW Convolution}

Our decomposition of the Greene-Wu convolution also allows other properties of the GW convolution easily derived. Perhaps one of the most important property of a function is its Hessian (matrix), which we discuss now. The definition of Hessian is in Definition \ref{def:hessian} (e.g., \cite{schoen1994lectures,petersen2006riemannian}). 


\begin{definition}[Hessian]
    \label{def:hessian}
    Consider an $n$-dimensional Riemannian manifold $(\M, g)$. For $f \in C^{\infty} (\M)$, define the Hessian of $f$ as $ H (f) : \mathfrak{X} ( \M )  \times \mathfrak{X} (\M ) \rightarrow C^\infty (\M)$ such that 
    \begin{align*} 
        H (f) \( X,Y \) := \nabla_{X} df (Y), \quad \forall X,Y \in \mathfrak{X} (\M), 
    \end{align*} 
    where $\nabla$ is the Levi-Civita connection. 
    Using Christoffel symbols, it holds that $$ H (f) (\partial_i, \partial_j ) = \partial_i \partial_j f - \Gamma_{ij}^k \partial_k f. $$ 
    The Hessian matrix of $f$ at $p$ is a map $ H_{p,f} : T_p \M \rightarrow T_p \M $, such that $ H_{p,f} (v) = \< \grad f \big|_p, v \> $ for any $ v \in T_p \M$.  
\end{definition} 

With the Levi-Civita connection, which is torsion-free, one can verify that $$ H (f) (X,Y) = H (f) (Y,X) . $$ 
For any $p \in \M$, with the normal coordinate specified by the exponential map $\Exp_p$, the Christoffel symbol vanishes at $p$. 
This means $H (f) (\partial_i, \partial_j) \Big|_p = \partial_i \partial_j f \big|_p$ in local normal coordinate. Similar to the Euclidean case, the eigenvalues of the Hessian matrix describes the curvature of the function. Next, we will show how the curvature of the function is affected after applying the GW convolution. 


In the study of the second order properties of the GW approximation, we will restrict our attention to kernels whose radial density function is uniform. More specifically, we will use $\alpha = 1$ and $  \Theta_{\mu, \alpha} (t) = \frac{1}{2 \mu} \Ind_{ \[ t \in \[ -\mu, \mu \] \] }$. Since one can construct other kernels by combining kernels with uniform radial density functions, investigating kernels with uniform radial density is sufficient for our purpose. With this in mind, we look at 
\begin{align}
    \wt{f}_{v}^\mu (q) = \frac{1}{2\mu} \int_{-\mu}^\mu f \( \Exp_q (t v_q) \) dt . 
\end{align}

    



The following theorem tells us how the geometry of the function and the geometry of the space are related after the GW convolution. 

\begin{theorem} 
    \label{thm:H-error} 
    Consider a complete Riemannian manifold $\M$ and a kernel with radial density $  \Theta_{\mu, \alpha} (t) = \frac{1}{2 \mu} \Ind_{ \[ t \in \[ -\mu, \mu \] \] } $. For any $p \in \M$, there exists $\mu_0 >0$ such that for all $\mu \in (0, \mu_0)$, it holds that 
    \begin{align*} 
        D_{uu} \wt{f}_{v}^\mu (p) 
        =& 
        \frac{1}{2 \mu} \int_{t \in [-\mu, \mu]} D_{ u_{\Exp_p (tv)} u_{\Exp_p (tv)} } f ( \Exp_p (tv) ) \nonumber \\
        &+  \frac{1}{2 \mu} \int_{-\mu}^\mu  \sum_{j = 1}^\infty \frac{ t^{2j} }{ (2j)! } \nabla_{ u}^2 \nabla_{v}^{2j} f (p) \, dt - \frac{1}{2 \mu} \int_{-\mu}^\mu \sum_{j = 1}^\infty \frac{ t^{2j} }{ (2j)!} \nabla_{v}^{2j} \nabla_{ u}^2 f (p) \, dt , \nonumber 
    \end{align*} 
    for all $v,u \in \S_p$. 
\end{theorem}

We will use the double exponential map \citep{gavrilov2007double} notation in the proof. The double exponential map (also written $\Exp_p$) is defined as $\Exp_p : T_p \M \times T_p \M  \rightarrow \M$ such that, when $v \in T_p \M $ is small enough ($ \Exp_p (v) \in U_p$), 
\begin{align}
    \Exp_p (v,u) = \Exp_{\Exp_p(v)} (u_{\Exp_p (v)}). 
\end{align}

With this notation, we are ready to prove Theorem \ref{thm:H-error}. 

\begin{proof}[Proof of Theorem \ref{thm:H-error}] 

    With this notion of double exponential map, we have 
    \begin{align}
         D_{uu} \wt{f}_{v}^\mu (p)  
        :=& \; 
        \lim_{\tau \rightarrow 0}  \frac{ \wt{f}_{v}^\mu ( \Exp_p ( \tau u  ) ) - 2 \wt{f}_{v}^\mu ( p ) + \wt{f}_{v}^\mu ( \Exp_p ( - \tau u  ) ) }{\tau^2} \nonumber  \\ 
        =& \; 
        \frac{1}{2 \mu} \int_{-\mu}^\mu \lim_{\tau \rightarrow 0}  \frac{ {f} ( \Exp_{ p } ( \tau u, t v  ) ) - 2 {f} ( \Exp_p (tv) ) + {f} ( \Exp_{ p } ( - \tau u, t v  ) ) }{\tau^2} \, dt  . \label{eq:def-duu} 
    \end{align} 
    
    
    Let $\varphi_u ( t, \tau ) = f ( \Exp_p (tv, \tau u) ) $, and let $$
        \psi_{u,v} ( \tau,t ) = f ( \Exp_p (  \tau u , t v ) ) -  f ( \Exp_p (  t v, \tau u ) ) + f ( \Exp_p (  - \tau u , t v ) )  - f ( \Exp_p (  t v, - \tau u ) ) . 
    $$  
    
    We can then write (\ref{eq:def-duu}) as 
    \begin{align} 
        D_{uu} \wt{f}_{v}^\mu (p) 
        =& 
        \frac{1}{2 \mu} \int_{t \in [-\mu, \mu ]} \lim_{\tau \rightarrow 0} \frac{ \varphi_u ( t, \tau ) - 2 \varphi_u (t, 0 ) + \varphi_u ( t, - \tau ) + \psi_{u,v} (\tau, t) }{\tau^2} . \label{eq:sec-in-phi}
    \end{align} 
    
    For the terms involving $\varphi_u$, one has 
    \begin{align}
        \lim_{\tau \rightarrow 0} \frac{\varphi_u (t,\tau) - 2 \varphi_u (t,0) + \varphi_u (t,-\tau)}{\tau^2} = \frac{ d^2 \varphi_u ( t, \tau )  }{ d \tau^2 } \Big|_{\tau = 0} = D_{ u_q u_q } f ( q ),  \label{eq:for-varphi}
    \end{align}  
    where $q = \Exp_p ( tv )$ and $u_q = \P_{p \rightarrow q} (u)$.

    
    
    
    For $\psi_{u,v} (\tau,t )$, we consider the term $ f (\Exp_p (u,v)) $ for a smooth function $f$ and small vectors $ u, v \in T_p \M $. 
    For any $p$, $v \in T_p \M$ and $q \in U_p$, define $h_v^{(j)} (q) = \nabla_{v_q}^j f (q)$. We can Taylor expand $ h_v (\Exp_p (u) ) $ by 
    \begin{align*} 
        h_v^{(j)} (\Exp_p (u) ) 
        =& \; 
        h_v^{(j)} (\Exp_p ( t u) ) \big|_{t=1} \\
        =& \; 
        \sum_{i=0}^\infty \frac{1}{i!} \frac{d^i}{ dt^i } h_v^{(j)} (\Exp_p ( t u) ) \\
        =& \; 
        \sum_{i=0}^\infty \frac{1}{i!} \nabla_u^i h_v^{(j)} ( p ) \\
        \overset{(a)}{=}& \; 
        \sum_{i=0}^\infty \frac{1}{i!} \nabla_u^i \nabla_v^j f ( p ) .  
    \end{align*} 
    
    
    Thus we have, for any $p$, $u,v \in T_p\M$ of small norm, 
    \begin{align*}
        f \( \Exp_{ p } \( u,v \) \) 
        =& \;
        f \( \Exp_{ \Exp_p (u) } \( v_{\Exp_p (u)} \) \) \\
        =& \;
        \sum_{j=0}^\infty \frac{1}{j!} \nabla_{ v_{\Exp_p (u)} }^j f ( \Exp_p (u) ) \\
        =& \;
        \sum_{j=0}^\infty \frac{1}{j!} h_{v}^{(j)} (\Exp_p (u) ) \\
        =& \; 
        \sum_{j=0}^\infty \frac{1}{j!} \sum_{i=0}^\infty \frac{1}{i!} \nabla_u^i \nabla_v^j f ( p ) , 
    \end{align*} 
    where the last equality uses $ (a) $. 
    
    
    
    Thus we have 
    \begin{align*} 
        &\; f (\Exp_p ( \tau u, t v  ))  + f (\Exp_p ( - \tau u, t v )) - f (\Exp_p ( t v, \tau u  )) - f (\Exp_p ( t v, -\tau u ))  \\ 
        =& \; 
        \sum_{i=0}^\infty \sum_{j=0}^\infty \frac{1}{i! j!} \nabla_{\tau u}^i \nabla_{t v}^j f (p) + \sum_{i=0}^\infty \sum_{j=0}^\infty \frac{1}{i! j!} \nabla_{ - \tau u}^i \nabla_{t v}^j f (p) \\
        &- \sum_{i=0}^\infty \sum_{j=0}^\infty \frac{1}{i! j!}  \nabla_{t v}^j \nabla_{ \tau u}^i f (p) - \sum_{i=0}^\infty \sum_{j=0}^\infty \frac{1}{i! j!}  \nabla_{t v}^j \nabla_{ - \tau u}^i f (p) \\
        =& \; 
        \sum_{j=0}^\infty \frac{ \tau^2 }{2 j! } \nabla_{ u}^2 \nabla_{t v}^j f (p) + \sum_{j=0}^\infty \frac{ \tau^2 }{2 j!} \nabla_{ u}^2 \nabla_{t v}^j f (p) \\ 
        &- \sum_{j=0}^\infty \frac{ \tau^2 }{2 j!} \nabla_{t v}^j \nabla_{ u}^2 f (p) - \sum_{j=0}^\infty \frac{ \tau^2 }{2 j!} \nabla_{t v}^j \nabla_{ u}^2 f (p) + \mathcal{O} (\tau^3) . 
    \end{align*}

    Thus we have 
    \begin{align} 
        \lim_{\tau \rightarrow 0 } \frac{ \psi_{u,v} ( \tau, t ) }{ \tau^2 } 
        = 
        \sum_{j = 1}^\infty \frac{ t^{2j} }{ (2j)! } \nabla_{ u}^2 \nabla_{v}^{2j} f (p) - \sum_{j = 1}^\infty \frac{ t^{2j} }{ (2j)!} \nabla_{v}^{2j} \nabla_{u}^2 f (p) + Odd (t), \label{eq:for-psi} 
    \end{align}  
    where $Odd (t)$ denotes terms that are odd in $t$. 
    
    Combining equations (\ref{eq:sec-in-phi}), (\ref{eq:for-varphi}), (\ref{eq:for-psi}) and that
    \begin{align*} 
        \int_{-\mu}^\mu Odd (t ) \, dt = 0
    \end{align*} 
    gives 
    \begin{align} 
        D_{uu} \wt{f}_{v}^\mu (p) 
        =& 
        \frac{1}{2 \mu} \int_{t \in [-\mu, \mu]} D_{ u_q u_q } f ( q ) \nonumber \\ 
        &+  \frac{1}{2 \mu} \int_{-\mu}^\mu  \sum_{j = 1}^\infty \frac{ t^{2j} }{ (2j)! } \nabla_{ u}^2 \nabla_{v}^{2j} f (p) \, dt -  \frac{1}{2 \mu} \int_{-\mu}^\mu \sum_{j = 1}^\infty \frac{ t^{2j} }{ (2j)!} \nabla_{v}^{2j} \nabla_{ u}^2 f (p) \, dt , \nonumber 
    \end{align} 
    where $q = \Exp_p (tv)$. 

\end{proof}


Theorem \ref{thm:H-error} says that when the space is not flat, the GW convolution changes the curvature of the function (Hessian matrix) in a non-trivial way. 
A more concrete example of this fact is in Corollary \ref{cor}. 

\begin{corollary} 
    \label{cor} 
    Consider a kernel with radial density function $\Theta_{\mu, \alpha} (t) = \frac{1}{2 \mu} \Ind_{ \[ t \in \[ -\mu, \mu \] \] }$. For a smooth function $ f $ defined over a complete Riemannian manifold $\M$ and any $p \in \M$, there exists $\mu_0 >0 $, such that for all $\mu \in (0, \mu_0)$, it holds that 
    \begin{align*} 
        \lambda_{\min} \( H_{p, \wt{f}^\mu} \) 
        \ge& \; 
        \frac{1}{2\mu} \E_{v \sim \S_p } \[ \int_{-\mu}^\mu \lambda_{\min} \( H_{\Exp_p ( tv ), f } \) dt \] \\ 
        & + \min_{u \in \S_p } 
        \frac{1}{2 \mu } \E_{v \sim \S_p} \[ \int_{-\mu}^\mu  \sum_{j = 1}^\infty \frac{ t^{2j} }{ (2j)! } \nabla_{ u}^2 \nabla_{v}^{2j} f (p) \, dt - \int_{-\mu}^\mu \sum_{j = 1}^\infty \frac{ t^{2j} }{ (2j)!} \nabla_{v}^{2j} \nabla_{ u}^2 f (p) \, dt \]. 
    \end{align*} 
\end{corollary} 

Since the minimum eigenvalue of the Hessian matrix measures the degree of convexity of a function, this corollary establish a quantitative description of how the GW convolution would affect the degree of convexity. 


\begin{proof}[Proof of Corollary \ref{cor}] 
    For any smooth function $f$ and point $p$, one has  \citep[e.g.,][]{absil2009optimization} 
    \begin{align}
        u^\top H_{p,f} u 
        := 
        \< H_{p,f} (u) , u \> 
        = 
        D_{uu} f (p) . \nonumber 
    \end{align}
    and thus
    \begin{align}
        \lambda_{\min} \( H_{p, \wt{f}^\mu } \) 
        = 
        \min_{u \in \S_p } u^\top H_{p, \wt{f}^\mu } u 
        = 
        \min_{u \in \S_p } D_{uu} \wt{f}^\mu (p) . 
    \end{align} 
    Thus we have 
    \begin{align*} 
        &\; \min_{u \in \S_p } D_{uu} \wt{f}^\mu (p) \\ 
        \ge& \; 
        \frac{1}{2 \mu } \E_{v \sim \S_p } \[ \int_{ \mu }^\mu \min_{w \in \S_{\Exp_p (tv)} } D_{w w } f \( \Exp_p (tv) \) dt \] \\ 
        &+ 
        \min_{u \in \S_p } \frac{1}{2 \mu } \E_{v \sim \S_p} \[ \int_{-\mu}^\mu  \sum_{j = 1}^\infty \frac{ t^{2j} }{ (2j)! } \nabla_{ u}^2 \nabla_{v}^{2j} f (p) \, dt - \int_{-\mu}^\mu \sum_{j = 1}^\infty \frac{ t^{2j} }{ (2j)!} \nabla_{v}^{2j} \nabla_{ u}^2 f (p) \, dt \] . 
    \end{align*} 
    Since $ \lambda_{\min} \( H_{p, \wt{f}^\mu } \) = \min_{u \in \S_p } u^\top H_{p, \wt{f}^\mu } u = \min_{u \in \S_p } D_{uu} \wt{f}^\mu (p)  $ for any $p$ and $f$, it holds that 
    \begin{align*}
        \lambda_{\min} \( H_{p, \wt{f}^\mu} \) 
        \ge& \; 
        \frac{1}{2\mu} \E_{v \sim \S_p } \[ \int_{-\mu}^\mu \lambda_{\min} \( H_{\Exp_p ( tv ), f } \) dt \]  \\ 
        &\; + 
        \min_{u \in \S_p } \frac{1}{2 \mu } \E_{v \sim \S_p} \[ \int_{-\mu}^\mu  \sum_{j = 1}^\infty \frac{ t^{2j} }{ (2j)! } \nabla_{ u}^2 \nabla_{v}^{2j} f (p) \, dt - \int_{-\mu}^\mu \sum_{j = 1}^\infty \frac{ t^{2j} }{ (2j)!} \nabla_{v}^{2j} \nabla_{ u}^2 f (p) \, dt \]. 
    \end{align*}
\end{proof}

\section{Estimating Gradient over Riemannian Manifolds}  

From our formulation of the GW approximation, one can obtain tighter bounds for gradient estimation over Riemannian manifolds, for geodesically $L_1$-smooth functions defined as follows. 

\begin{definition} 
    Let $p\in \M$. A function $f: \M \rightarrow \R$ is called geodesically $L_1$-smooth near $p$ if 
    \begin{align}
        \left\| \P_{ p \rightarrow q } \( \grad f\big|_p \) -  \grad f\big|_q  \right\|_q  \le L_1  d_p (p,q),  \quad \forall q \in U_p ,  
    \end{align}
    where 
    $\grad f \big|_p$ is the gradient of $f$ at $p$. If $f$ is $ L_1 $-smooth near $p$ for all $p \in \M$, we say $f$ is $ L_1 $-smooth over $\M$. 
\end{definition} 

For a geodesically $L_1$-smooth function, one has the following theorem that bridges the gradient at $p$ and the zeroth-order information near $p$. 

\begin{theorem}
    \label{thm:grad-error} 
    Fix $\alpha > 0$. Let $f$ be geodesically $L_1$-smooth near $p$.
    It holds that 
    \begin{align} 
        \left\| \grad f\big|_p - \frac{ n }{ \mu \alpha^2 }  \E_{ v \sim \S_p (\alpha) } \[ f \( \Exp_p (\mu v) \) v \] \right\|_p \le \frac{L_1 n \alpha \mu}{2 }, 
    \end{align}
    where $\S_p (\alpha) $ is the origin-centered sphere in $ T_p \M $ of radius $ \alpha $. 
\end{theorem}

This theorem provides a very simple and practical gradient estimation method: 
At point $p \in \M$, we can select small numbers $\mu$ and $\alpha$, and uniformly sample a vector $v \sim \S_p (\alpha)$. With these $\mu, \alpha, v$, the random vector  
\begin{align} 
    \wh{\grad} f \big|_p  (v) = \frac{ n }{ 2 \alpha^2 \mu } \( f \(\Exp_p (\mu v) \) - f \( \Exp_p (- \mu v) \) \) \label{eq:ball-estimator} 
\end{align} 
gives an estimator of $\grad f \big|_p$.

One can also independently sample $v_1, v_2, v_3, \cdots, v_m$ uniformly from $\S_p$. Using these samples, an estimator for $ \grad f\big|_p$ is
\begin{align} 
    \wh{\grad} f \big|_p (v_1,v_2, \cdots, v_m) 
    = 
    \frac{ n }{2 m  \alpha^2 \mu }  \sum_{i=1}^m \( f \( \Exp_p ( \mu v_i ) \) - f \( \Exp_p ( -\mu v_i ) \) \) v_i . \label{eq:ensembled-estimator}
\end{align} 
Note that Theorem \ref{thm:grad-error} also provides an in expectation bound for the estimator in (\ref{eq:ensembled-estimator}). 
Before proving Theorem \ref{thm:grad-error}, we first present Propositions \ref{prop:smooth} and \ref{prop:inner-prod}, and Lemma \ref{lem:fundamental-calculus}. 

\begin{proposition}
    \label{prop:smooth}
    Consider a continuously differentiable function $f$ defined over a complete Riemannian manifold $\M$. If function $ f $ is geodesically $L_1$-smooth over $\M$, then it holds that 
    \begin{enumerate} 
        \item $ f \( \Exp_p (u) \) \le f ( p ) + \< \grad f \big|_p, u \> + \frac{L_1}{2} \| u \|^2 $ for all $p \in \M$ and $u \in T_p \M$; 
        \item 
        for any $v \in T_p \M$, $ \left\| \P_{0,1}^\gamma \( \grad f \big|_p \) - \grad f \big|_{\Exp_p (v)} \right\| \le L_1 \| v \| $, where $\gamma (t) := \Exp_p (tv)$. 
    \end{enumerate} 
\end{proposition} 

\begin{proof}
    For any $p \in \M$ and $u \in T_p \M$, let $q_0 = p, q_1, q_2, \cdots, q_N = \Exp_p ( u )$ be a sequence of points such that $ q_{i+1} \in U_{q_i} $ and $ q_{i+1} = \Exp_{q_i} ( \tau u_{q_i} ) $ for all $\tau > 0$ that is sufficiently small. Since the Levi-Civita  connection preserves the Riemannian metric, we know $ N \tau = 1 $. 
    At any $q_i$, we have 
    \begin{align} 
        &\< \grad f \big|_{q_{i}} , u_{q_i} \> \nonumber \\
        =&
        \< \grad f \big|_{q_{i}} - \P_{ q_{i-1} \rightarrow q_i } \( \grad f  \big|_{q_{i-1}} \) , u_{q_i} \> + \< \grad f \big|_{q_{i-1}} , u_{q_{i-1}} \> \label{eq:use-L-smoothness} \\
        \le& 
        \< \grad f \big|_{q_{i-1}} , u_{q_{i-1}} \> + L_1 \tau \| u \|^2 \nonumber \\
        \le& \cdots \nonumber \\
        \le& \< \grad f \big|_{p} , u \> + i L_1 \tau \| u \|^2 , \nonumber
    \end{align} 
    where (\ref{eq:use-L-smoothness}) uses the definition of geodesic $L_1$-smoothness. 
    
    Also, it holds that
    \begin{align}
        &f \( q_{i} \) - f \( q_{i-1} \) \nonumber \\
        =& 
        \int_{ 0 }^{ \tau } \< \grad f \big|_{ \Exp_{q_{i-1}} ( t u_{q_{i-1}} ) } , u_{ \Exp_{q_{i-1}} ( t u_{q_{i-1}} ) } \> d t .  \qquad (\text{by Lemma \ref{lem:fundamental-calculus}}) \label{eq:for-add-subtract}
    \end{align} 
    By the definition of geodesic $L_1$-smoothness and the Cauchy-Schwarz inequality, it holds that 
    \begin{align}
        &\int_{ 0 }^{ \tau } \< \grad f \big|_{ \Exp_{q_{i-1}} ( t u_{q_{i-1}} ) } , u_{ \Exp_{q_{i-1}} ( t u_{q_{i-1}} ) } \> d t -  \tau \<   \grad f \big|_{ q_{i-1} } , u_{ q_{i-1} } \> \nonumber \\
        \le&
        \int_{ 0 }^{ \tau } \left\|  \grad f \big|_{ \Exp_{q_{i-1}} ( t u_{q_{i-1}} ) } \hspace*{-2pt} - \hspace*{-2pt} \P_{ q_{i-1} \rightarrow \Exp_{q_{i-1}} ( t u_{q_{i-1}}) } \( \grad f \big|_{ q_{i-1} } \) \right\| \| u \|   \,  d t \nonumber \\
        \le& 
        \frac{ L_1 \tau^2 \| u \|^2 }{2} \label{eq:for-sub-diff}. 
    \end{align} 
    Adding and subtracting $\tau \<   \grad f \big|_{ q_{i-1} } , u_{ q_{i-1} } \> $ to (\ref{eq:for-add-subtract}) and using (\ref{eq:for-sub-diff}) gives 
    \begin{align*} 
        f \( q_{i} \) - f \( q_{i-1} \) \le  \frac{ L_1 \tau^2 \| u \|^2 }{2} + \tau \< \grad f \big|_{ q_{i-1} } , u_{ q_{i-1} } \>. 
    \end{align*} 
    
    Thus we have 
    \begin{align}
        f ( \Exp_p (u) ) - f (p) \le&  \sum_{i=1}^{N} f (q_{i}) - f (q_{i-1}) \nonumber \\ 
        \le& \sum_{i=1}^N \( \tau \< \grad f \big|_{ q_{i-1} } , u_{ q_{i-1} } \> + \frac{ L_1 \tau^2 \| u \|^2 }{2} \) \nonumber \\
        \le& 
        N\tau \< \grad f\big|_p , u \> + 
        \sum_{i=1}^N L_1 (i-1) \tau^2 \| u \|^2 + \frac{ L_1 N \tau^2 \| u \|^2 }{ 2 } \nonumber \\ 
        =& 
        N\tau \< \grad f\big|_p , u \> + \frac{ L_1 N^2 \tau^2 \| u \|^2 }{2} + O ( N \tau^2 ),  \label{eq:take-tau-lim}
    \end{align} 
    where the $O \( \cdot \)$ notation omits constants that do not depend on $N$ or $\tau$. 
    Since $N \tau = 1$ and (\ref{eq:take-tau-lim}) is true for arbitrarily small $\tau$, letting $\tau \rightarrow 0$ in (\ref{eq:take-tau-lim}) finishes the proof. 
    
    The second item can be proved in a similar way. Specifically, we can also find a sequence of points $q_0, q_1, \cdots, q_N $ on the curve of $\gamma(t) = \Exp_p (t v)$ with $q_0 = p $ and $q_N = \Exp_p (v)$, and repeatedly use the definition of geodesic $L_1$-smoothness. 
    
\end{proof} 
    
\begin{proposition} 
    \label{prop:inner-prod}
    For any vector $u \in T_p \M $, we have 
    \begin{align*} 
         \E_{ v \sim \S_p (\alpha) } \[ \< u, v \> v \] = \frac{\alpha^2}{n} u , 
         \quad \text{ and } \quad 
         \E_{ v \sim \S_p (\alpha) } \[ \< u, v \>^2 \] =  \frac{\alpha^2}{n} \| u \|^2. 
    \end{align*} 
\end{proposition} 

\begin{proof} 
    It is sufficient to consider $u = e_i$, where $ \{ e_i \}_i$ is the local canonical basis for $T_p \M$ with respect to $g_p$.
    For any $i,j \in \{1,2,\cdots, n \}$, 
    \begin{align*} 
         \( \E_{ v \sim \S_p (\alpha) } \[ \< e_i, v \> v \] \)_j = \E_{ v \sim \S_p (\alpha) } \[ v_i v_j \].  
    \end{align*} 
    Since $ \E_{ v \sim \S_p (\alpha) } \[ v_j | v_i = x \] = 0$ for any $x$ when $ i \neq j$, we have $\E_{ v \sim \S_p (\alpha) } \[ v_i  v_j \] = 0$. When $ i = j $, by symmetry, we have $\E_{ v \sim \S_p (\alpha) } \[ v_j^2 \] = \frac{\alpha^2 }{n} $ for all $j = 1,2,\cdots, n$. 
    
    We have shown $ \E_{v \sim \S_p (\alpha) } \[ \< u, v \> v \] = \frac{\alpha^2}{n} u $, on which taking inner product with $u$ shows that $\E_{v \sim \S_p (\alpha) } \[ \< u, v \>^2 \] = \frac{\alpha^2}{n} \| u \|^2 $. 
\end{proof} 

\begin{lemma}
    \label{lem:fundamental-calculus}
    Pick $p \in \M$ and $v \in T_p \M$. Let $ \gamma (t) = \Exp_p (t v)$ $(t \in [a,b])$ be a geodesic in $\M$. 
    We then have 
    \begin{align} 
        \int_a^b \< \grad f \big|_{\Exp_p (sv )}, \dot{\gamma}\big|_{s} \> ds = f ( \gamma (b) ) - f ( \gamma (a) ). 
    \end{align} 
\end{lemma} 

\begin{proof}
    
    Consider the function $f \circ \gamma : [a, b] \rightarrow \mathbb{R}$. By fundamental theorem of calculus, we have 
    \begin{align*} 
        f ( \gamma (b) ) - f ( \gamma (a) ) 
        &= \int_a^b \frac{d f \circ \gamma}{ dt } {d}t . 
    \end{align*}
    Also, the directional derivative over Riemannian manifolds can be computed by
    \begin{align*}
        \frac{d f \circ \gamma}{ dt } \Bigg|_{s} = \lim_{\tau \rightarrow 0} \frac{ f \circ \gamma (s + \tau ) }{ \tau } = \< \grad f \big|_{\Exp_p (sv )}, \dot{\gamma}\big|_{s} \> . 
    \end{align*} 
    Combining the above results finishes the proof. 
\end{proof}

Lemma \ref{lem:fundamental-calculus} can be viewed as a corollary of the fundamental theorem of calculus, or a consequence of Stokes' theorem. If one views the curve $\gamma$ as a manifold with boundary ($\gamma (a)$ and $\gamma (b)$), and $\dot{\gamma}$ as a vector field parallel to $\gamma$, then applying Stokes' theorem proves Lemma \ref{lem:fundamental-calculus}.


We are now ready to prove Theorem \ref{thm:grad-error}, using our decomposition trick with the radial density function being $ \Theta_{\mu, \alpha} (x) = \frac{1}{2 \mu} \Ind_{ \[ x \in \[ -\mu, \mu \] \] } $. 

\begin{proof}[Proof of Theorem \ref{thm:grad-error}] 

From the definition of directional derivative, we have 
\begin{align} 
    \< \grad \wt{f}_{v}^\mu \big|_p , v \>_p &=  \frac{1}{2\mu} \lim_{ \tau \rightarrow 0 } \hspace{-2pt} \frac{ \int_{-\mu}^\mu \hspace{-2pt} \( \hspace{-2pt} f \hspace{-2pt} \( \Exp_{ \Exp_p (\tau v ) } \hspace{-2pt} \( t v_{\Exp_p (\tau v)} \) \hspace{-2pt} \) \hspace{-2pt} - \hspace{-2pt} f \( \Exp_{p } \hspace{-2pt} \( t v \) \) \hspace{-2pt} \) \hspace{-2pt} dt }{ \tau } \nonumber \\ 
    &=
    \frac{1}{2\mu} \int_{-\mu}^\mu \lim_{ \tau \rightarrow 0 } \frac{ f \( \Exp_p  \( \tau v + t v \) \) - f \( \Exp_{ p } \( t v \) \)  }{ \tau }  dt \nonumber \\ 
    &= 
    \frac{1}{2\mu} \int_{-\mu}^\mu \< \grad f \big|_{ \Exp_p \( t v \)  } , v_{\Exp_p (tv)} \>_{\Exp_p (tv)} dt  \label{eq:above}
\end{align} 

For simplicity, let $\gamma_v (t) = \Exp_p (t v)$ ($t \in [-\mu, \mu]$) for any $v \in T_p (\M)$. 
Since $f$ is geodesically $L_1$-smooth, we have, for any $p$ and $v \in T_p \M$,  
\begin{align}
    &\< \grad {f} \big|_p - \grad \wt{f}_{v}^\mu \big|_p , v \>_p \nonumber \\
    =&
    \hspace{-2pt} \< \hspace{-1pt} \grad {f} \big|_p , v \hspace{-2pt} \>_p \hspace{-2pt} - \hspace{-2pt} \frac{ \int_{-\mu}^\mu \< \grad f \big|_{ \Exp_p \( t v \)  } , v_{\Exp_p (tv)} \>_{\Exp_p (tv)} dt }{2 \mu}  \label{eq:use-above} \\
    =&
    \frac{1}{2\mu} \int_{-\mu}^\mu \( \< \P_{0,t}^{\gamma_v} \( \grad {f} \big|_p\) , v_{\Exp_p (tv)} \>_{\Exp_p (tv)} \) dt \nonumber \\
    &-  \frac{1}{ 2\mu}\int_{-\mu}^\mu \( \< \grad f \big|_{ \Exp_p \( t v \)  } , v_{\Exp_p (tv)} \>_{\Exp_p (tv)} \) dt \label{eq:preserve} \\
    =&
    \frac{1}{2\mu} \int_{-\mu}^\mu \( \< \P_{0,t}^{\gamma_v} \( \grad {f} \big|_p\) - \grad f \big|_{ \Exp_p \( t v \)  }, v_{\Exp_p (tv)} \>_{ \Exp_{p} (tv) } \) dt \nonumber \\ 
    \le&
    \frac{1}{2\mu} \int_{-\mu}^\mu L_1 |t| \alpha^2 dt = \frac{ L_1 \alpha^2 \mu}{2} \label{eq:lip}, 
\end{align} 
where (\ref{eq:use-above}) uses (\ref{eq:above}), (\ref{eq:preserve}) uses the fact that the parallel transport preserves the Riemannian metric, and (\ref{eq:lip}) uses Proposition \ref{prop:smooth} and the Cauchy-Schwarz inequality.

By Proposition \ref{prop:inner-prod}, we have 
\begin{align*}
    \E_{v \sim \S_p (\alpha) } \[ \< \grad f \big|_p, v \>_p v \] = \frac{\alpha^2}{n} \grad f \big|_p, 
\end{align*}  
and thus from (\ref{eq:lip}), we have 
\begin{align}
    &\left\|  \E_{v \sim \S_p (\alpha) } \[ \< \grad \wt{f}_{v}^\mu \big|_p, v \> v \] - \frac{ \alpha^2 }{ n } \grad f \big|_p \right\| \nonumber \\ 
    =& \left\|  \E_{v \sim \S_p (\alpha) } \[ \< \grad \wt{f}_{v}^\mu \big|_p, v \> v - \< \grad f \big|_p, v \> v \] \right\| \nonumber \\ 
    =& \left\|  \E_{v \sim \S_p (\alpha) } \[ \< \grad \wt{f}_{v}^\mu \big|_p - \grad f \big|_p, v \> v \] \right\| \nonumber \\ 
    \le& \frac{ L_1 \alpha^3 \mu}{2}.  \label{eq:est-error} 
\end{align}


Applying Lemma \ref{lem:fundamental-calculus} and the above results gives 
\begin{align}
    & \E_{v \sim \S_p (\alpha) } \[ \< \grad \wt{f}_{v}^\mu \big|_p , v \>_p v \] \nonumber \\ 
    =&  
    \E_{v \sim \S_p (\alpha) } \[ \frac{1}{2 \mu} \int_{- \mu}^\mu \< \grad f \big|_{ \Exp_p \( t v \) } , v_{\Exp_p ( t v ) } \>_{\Exp_p ( t v ) } dt \; v \] & (\text{by Eq. \ref{eq:above}}) \nonumber \\ 
    =& 
    \frac{1}{2\mu} \E_{v \sim \S_p (\alpha) } \[ \( f \( \Exp_p (\mu v) \) - f \( \Exp_p (\mu v) \) \) v \]  & (\text{by Lemma \ref{lem:fundamental-calculus}}) \nonumber \\ 
    =& 
    \frac{1}{\mu} \E_{ v \sim \S_p (\alpha) } \[ f \( \Exp_p (\mu v) \) v \] . \nonumber 
\end{align}

Combining the above result with (\ref{eq:est-error}) gives 
\begin{align}
    \left\| \grad f\big|_p - \frac{ n }{ \mu \alpha^2 } \E_{ v \sim \S_p (\alpha) } \[ f \( \Exp_p (\mu v) \) v \] \right\| \le \frac{L_1 n \alpha \mu}{2 }. 
\end{align} 

\end{proof} 





The variance of the estimator (\ref{eq:ball-estimator}) is shown in Theorem \ref{thm:variance}. 

\begin{theorem}
    \label{thm:variance}
    If $f$ is geodesically $L_1$-smooth, then it holds that 
    \begin{align*} 
        &\E_{v_1, v_2, \cdots, v_m \overset{i.i.d.}{\sim} \S_p (\alpha) } \[ \left\| \wh{\grad} f\big|_p \(v_1, v_2, \cdots, v_m\)  \big|_p \right\|^2 \] \\
        \le&
        \frac{n^2}{m} \( \frac{ L_1 \alpha \mu }{2} + \left\| \grad f \big|_p \right\| \)^2 \\
        &+ \frac{ m-1 }{m} \( \left\| \grad f \big|_p \right\|^2 +  \left\|  \grad f \big|_p  \right\|  L_1 n^{3/2} \alpha \mu  + \frac{ L_1^2 n^3 \alpha^2 \mu^2 }{ 4  } \).  
    \end{align*} 
\end{theorem}

\begin{proof}
   
    By Lemma \ref{lem:fundamental-calculus} and that $f $ is geodesically $L_1$-smooth, we have, for any $v \in \S_p (\alpha)$,

    \begin{align}
        & \left| f \( \Exp_p (\mu v ) \) - f \( \Exp_p (-\mu v) \) \right|  \nonumber \\ 
        =& 
        \left| \int_{-\mu}^\mu \< \grad f \big|_{\Exp_p (t v )} , v_{\Exp_p (tv) } \> dt - \int_{-\mu}^\mu \< \grad f \big|_p, v \> dt + 2 \mu \< \grad f \big|_p, v \> \right| \nonumber \\ 
        \le& 
        \int_{-\mu }^\mu  \left| \< \grad f \big|_{\Exp_p (t v )} - \P_{p \rightarrow \Exp_p (t v)} \( \grad f \big|_p \), \P_{p \rightarrow \Exp_p (t v)} (v) \>\right| dt \nonumber \\ &+ \left| 2 \mu \< \grad f \big|_p , v \> \right| \nonumber \\ 
        \le& 
        L_1 \int_{-\mu}^\mu  |t| \| v \|^2 dt + 2 \alpha \mu \left\| \grad f \big|_p  \right\| \nonumber \\ 
        \le& L_1 \alpha^2 \mu^2 + 2 \alpha \mu \left\| \grad f \big|_p  \right\| . 
        \label{eq:l1-smooth-b} 
    \end{align} 
    
    By (\ref{eq:l1-smooth-b}), it holds that, 
    \begin{align} 
        \E \[ \left\| \wh{\grad} f\big|_p \(v\) \right\|^2 \] 
        =& 
        \E \[ \frac{n^2}{4\mu^2 \alpha^4 } \( f \( \Exp_p (\mu v) \) - f \( \Exp_p (-\mu v) \)  \)^2 \| v \|^2 \] \nonumber \\ 
        \le& 
        {n^2} \( \frac{ L_1 \alpha \mu }{2} + \left\| \grad f \big|_p \right\| \)^2 
        \label{eq:smooth-b}  
    \end{align}
    
    By Theorem \ref{thm:grad-error}, we have 
    \begin{align*} 
        \E \[ \wh{\grad} f \big|_p (v_i) \]
        = 
        \grad f \big|_p + O \( \frac{ L_1  n \alpha \mu }{ 2 } \mathbf{1} \), 
    \end{align*} 
    where $ \mathbf{1}$ is the all-one vector and the big-O notation here means that, with respect to any coordinate system, each entry in $ \E \[ \wh{\grad} f \big|_p (v_i) \] $ and $\grad f \big|_p $  differs by at most  $ \frac{  L_1 n \alpha \mu }{ 2 } $. 
    
    Since $v_1, v_2, \cdots, v_m$ are mutually independent, we have, for any $i \neq j$, 
    \begin{align*} 
        &\E \< \wh{\grad} f \big|_p (v_i) , \wh{\grad} f \big|_p (v_j) \> \\
        =& 
        \< \grad f \big|_p + O \( \frac{ L_1 n \alpha \mu }{2 } \mathbf{1} \) , \grad f \big|_p + O \(  \frac{ L_1 n \alpha \mu }{2 } \mathbf{1} \)  \> \\ 
        \le& 
        \left\| \grad f \big|_p \right\|^2 +  \left\|  \grad f \big|_p  \right\|  L_1 n^{3/2} \alpha \mu  + \frac{ L_1^2 n^3 \alpha^2 \mu^2 }{ 4 }, 
    \end{align*} 
    where the last inequality uses the Cauchy-Schwartz inequality. 
    Thus by expanding out all terms we have 
    \begin{align*} 
        &\E \[ \left\| \wh{\grad} f\big|_p \(v_1, v_2, \cdots, v_m\) \right\|^2 \]  \\ 
        =& 
        \frac{1}{m^2} \E \[ \sum_{i=1}^m \left\| \wh{\grad} f \big|_p \( v_i \) \right\|^2 \] \\ 
        &+ \sum_{1 \le i,j \le m: i \neq j} \E  \< \grad f \big|_p + O \( \frac{ L_1 n \alpha \mu }{2 } \mathbf{1} \) , \grad f \big|_p + O \(  \frac{ L_1 n \alpha \mu }{2 } \mathbf{1} \) \> \\ 
        \le& 
        \frac{n^2}{m} \( \frac{ L_1 \alpha \mu }{2} + \left\| \grad f \big|_p \right\| \)^2 \\
        &+ \frac{ m-1 }{m} \( \left\| \grad f \big|_p \right\|^2 +  \left\|  \grad f \big|_p  \right\|  L_1 n^{3/2} \alpha \mu  + \frac{ L_1^2 n^3 \alpha^2 \mu^2 }{ 4  } \). 
    \end{align*}  
    
    
\end{proof}



    


\subsection{Previous Methods for Riemannian Gradient Estimation}


Previously, \cite{nesterov2017random,li2020stochastic} have introduced gradient estimators using the following sampler. The estimator is 
\begin{align}
    \wh{\grad } f  \big|_{p} (v_1, v_2, \cdots, v_m) = \frac{1}{2 m \mu } \sum_{i=1}^m \( f \( \Exp_p ( \mu v_i) \) - f \( \Exp_p (- \mu v_i) \) \) v_i, \label{eq:g-estimator} 
\end{align}
where $v_i \overset{i.i.d.}{\sim} \mathcal{N} \( 0, I \)$ are Gaussian vectors, and $\mu > 0$ is a parameter controlling the width of the estimator. We use this Gaussian sampler as a baseline for empirical evaluations. For this estimator,  \cite{li2020stochastic} studied its properties over manifolds embedded in Euclidean space. Their result is in Proposition \ref{prop:prev-garbage}. 

\begin{proposition}[\cite{li2020stochastic}]
    \label{prop:prev-garbage}
    Let $\M$ be a manifold embedded in a Euclidean. If $f: \M \rightarrow \R $ is geodesically $L_1$-smooth over $\M$, then it holds that 
    \begin{align} 
        \left\| \E_{v_i \overset{i.i.d.}{\sim} \mathcal{N} (0, I) } \[ \wh{\grad}f \big|_p (v_1,v_2,\cdots, v_m) \] - \grad f \big|_p \right\| \le \frac{L_1 \( d + 3 \)^{3/2} \mu  }{2}
    \end{align}
    where $ \wh{\grad}f \big|_p (v_1, v_2, \cdots, v_m) $ is the estimator in (\ref{eq:g-estimator}). 
\end{proposition}  

To fairly compare our method (\ref{eq:ball-estimator}) and this previous method (\ref{eq:g-estimator}), one should set $\alpha = \sqrt{n}$ for our method, as pointed out in Remark \ref{remark:compare}. 

\begin{remark}
    \label{remark:compare}
    If one picks $\alpha = \sqrt{n}$ in Theorem \ref{thm:grad-error}, one can obtain the error bound $\frac{L_1 n^{3/2} \mu }{2}$ for our method, and the error bound for the previous method is $\frac{L_1 (n+3)^{3/2} \mu }{2}$ (Proposition \ref{prop:prev-garbage}). Why should we pick $\alpha = \sqrt{n}$ in Theorem \ref{thm:grad-error}? This is because when $\alpha = \sqrt{n}$, we have $\E_{v \sim \S (\alpha) }[ \| v \|^2 ] = \E_{v \sim \mathcal{N} (0, I) } [\| v \|^2 ] = n$. In words, when $\alpha = \sqrt{n}$, the random vectors used in our method (\ref{eq:ball-estimator}) and the random vector used in the previous method (\ref{eq:g-estimator}) have the same squared norm in expectation. This leads to a fair comparison for the bias, while holding the second moment of the random vector exactly the same. 
\end{remark}

\section{Empirical Studies} 



In this section, we empirically study the spherical estimator in (\ref{eq:ensembled-estimator}), in comparison with the Gaussian estimator in (\ref{eq:g-estimator}). 
The two methods (\ref{eq:ensembled-estimator}) and (\ref{eq:g-estimator}) are compared over three manifolds: 1. the Euclidean space $\R^n$, 2. the unit sphere $\S^{n}$, and 3. the surface specified by the equation $h (x) = \frac{1}{2} \sum_{i=1}^{n/2} x_i^2 - \frac{1}{2} \sum_{i = n/2 + 1}^n x_i^2$. The evaluation is carried out at three different point, one for each manifold. The three manifolds, and corresponding points for evaluation, are listed in Table \ref{tab:exp}. 
Two test functions are used: a. the linear function $ f (x) = \sum_{i=1}^n x_i $, and b. the function $f (x) = \sum_{ i=1 }^n  \sin \( x_i - 1 \) $. 
The two test functions are listed in Table \ref{tab:func}.

\begin{table}[H]
    \centering
    \begin{tabular}{c|c|c|c}
        Label & Manifold & Point $p$  & Exponential map $\Exp_p (v)$ \\  \hline \hline 
        1 & $\R^n $ & $ 0$ & $ v$ \\  
        2 & $\S^{n} $ & $(1,0,\cdots, 0)$ &  $ p \cos (\| v \| ) + \frac{ v }{ \| v \| } \sin \( \| v \| \)$ \\ 
        3 & \makecell{ $(x, h(x)) $, \\ $x \in \R^n$} & $ 0 $ &  $\( v,  \( \frac{\sqrt{ 1 + \sum_{i=1}^n v_i^4  }  }{2} + \frac{ \sinh^{-1} \( \sqrt{ \sum_{i=1}^n v_i^4 } \) }{ \sqrt{ \sum_{i=1}^n v_i^4 } } \) h (v) \) $ \\ 
    \end{tabular}
    \caption{Manifolds used for experiments. All manifolds are of dimension $n$. The coordinate system of the ambient space is used in all settings. } 
    \label{tab:exp}
\end{table} 

\begin{table}[H]
    \centering
    \begin{tabular}{c|c}
        Label & Function \\  \hline \hline 
        a & $f (x) = \sum_{i=1}^n x_i $ \\  
        b & $f (x) = \sum_{i=1}^n \sin (x_i - 1)$ \\ 
    \end{tabular}
    \caption{Functions used for experiments. The coordinate system of the ambient Euclidean space is used for the computations. } 
    \label{tab:func}
\end{table} 

The three manifolds in Table \ref{tab:exp} and the two functions in Table \ref{tab:func} together generate 6 settings. We use 1a, 1b, 2a, 2b, 3a, 3c to label these 6 setting, where 1a refers to the setting over manifold 1 using function a, and so on. The experimental results are summarized in Figures \ref{fig:1a}-\ref{fig:error-3b}. In all figures, ``G'' on the $x$-axis stands for the previous method using Gaussian estimator (\ref{eq:g-estimator}), and ``S'' on the $x$-axis stands for our method using the spherical estimator (\ref{eq:ensembled-estimator}). 
To fairly compare (\ref{eq:g-estimator}) and (\ref{eq:ensembled-estimator}), we use $\alpha = \sqrt{n} $ for all spherical estimators so that for both the spherical estimator and the Gaussian estimator, one has $\E \[ \| v \|^2 \] = n$. 
The figure captions specify the settings used. For example, setting 1a is for function $a$ in Table \ref{tab:func} over manifold 1 in Table \ref{tab:exp}. Below each subfigure, the values of $n$ and $\mu$ are the dimension of the manifold and the parameter used in the estimators (both G and S). 


Since $\E \[ \| v \|^2 \]$ are set to the same value for both G and S, we use 
\begin{align}
    \left\|  \wh{\grad} f \big|_p (v_1, v_2, \cdots, v_m) - \grad f \big|_p \right\| \label{eq:bias-criteria}
\end{align}
to measure the error (bias) of the estimator, and 
\begin{align}
    \frac{1}{m} \sum_{i=1}^m \left\| \wh{\grad} f\big|_p (v_i) \right\|^4 \label{eq:robustness-criteria}
\end{align}
to measure robustness of the estimator. For both the estimation error and the robustness, smaller values mean better performance. For both ``G'' and ``S'' in all figures, each violin plot summarizes $100$ values of (\ref{eq:bias-criteria}) or (\ref{eq:robustness-criteria}), where we use $m = 100$ for all of them. 
Take G and S in Figure \ref{fig:1a} as an example. We compute (\ref{eq:robustness-criteria}) for 100 times for G and $100$ estimations for S, all with $m = 100$. Then we use violin plots to summarize these 100 values of (\ref{eq:robustness-criteria}) for G and the 100 values of (\ref{eq:robustness-criteria}) for S. 
       

Figures \ref{fig:1a}-\ref{fig:3b} plot the robustness of the two estimators $G$ and $S$, and Figures \ref{fig:error-1a}-\ref{fig:error-3b} show the corresponding errors. As shown in the figures, in all of the settings, the spherical estimator S is much more robust than the Gaussian estimator G, while achieving the same level of estimation error. 

\newcommand\SCALE{0.325} 


\begin{figure}[H]
    \centering
    \includegraphics[scale =  \SCALE]{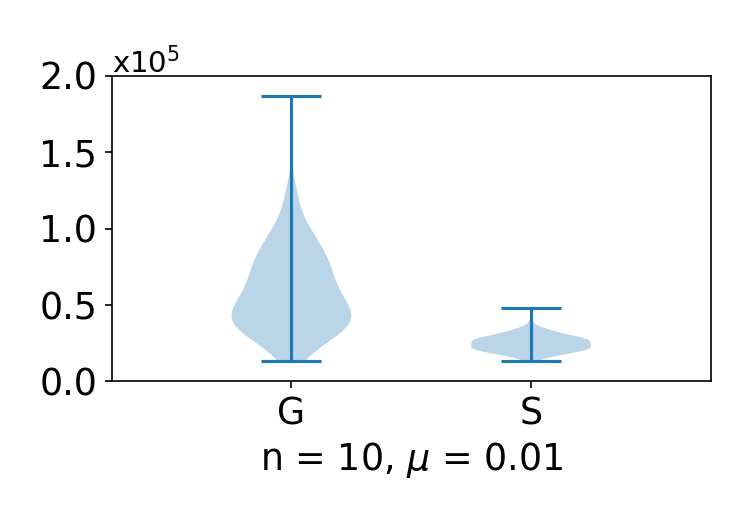}  
    \includegraphics[scale =  \SCALE]{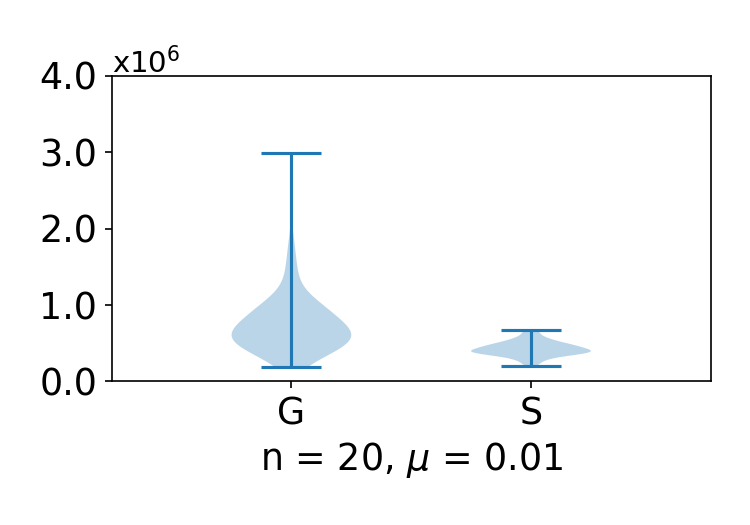}  
    \includegraphics[scale =  \SCALE]{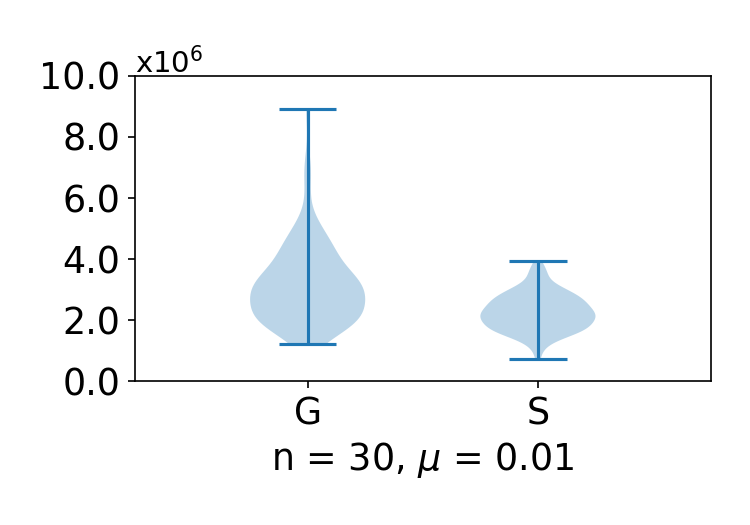} \\
    \includegraphics[scale =  \SCALE]{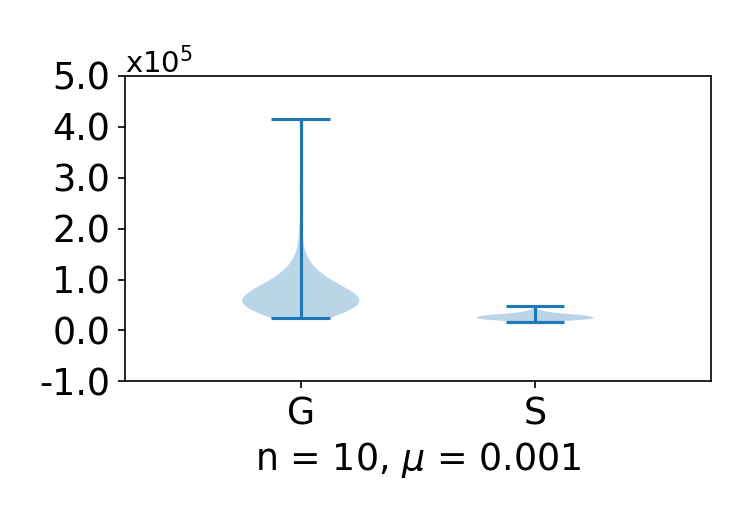}  
    \includegraphics[scale =  \SCALE]{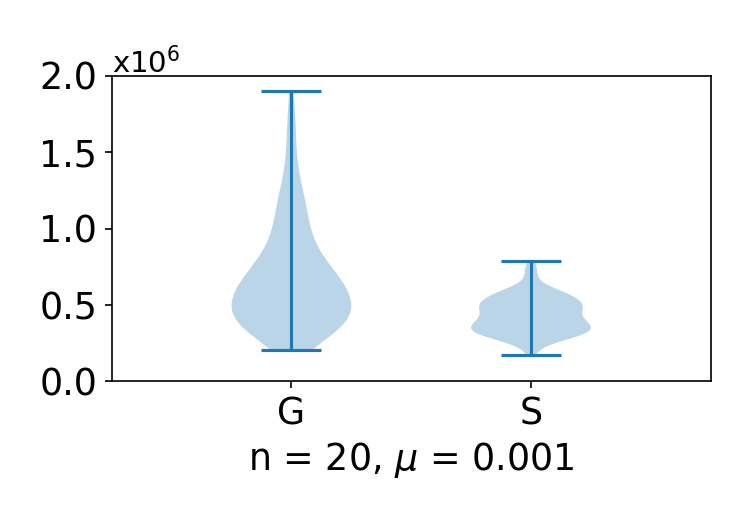}  
    \includegraphics[scale =  \SCALE]{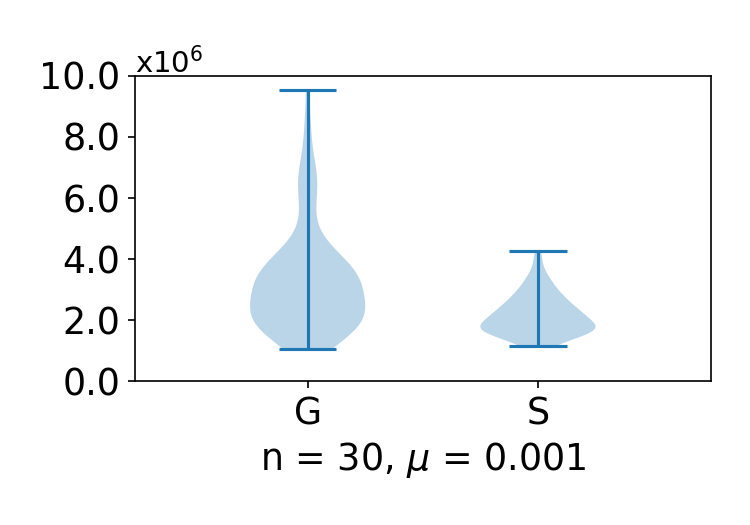} \vspace*{-0.3cm}
    \caption{Plot of $ \frac{1}{m} \sum_{i=1}^m \left\| \wh{\grad} f\big|_p (v_i) \right\|^4 $ for Setting 1a.
    \label{fig:1a} }
\end{figure}

\begin{figure}[H]
    \centering
    \includegraphics[scale =  \SCALE]{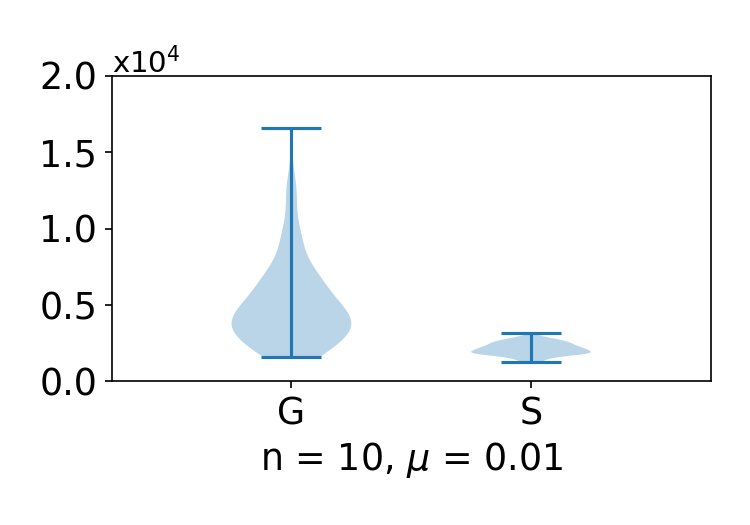}  
    \includegraphics[scale =  \SCALE]{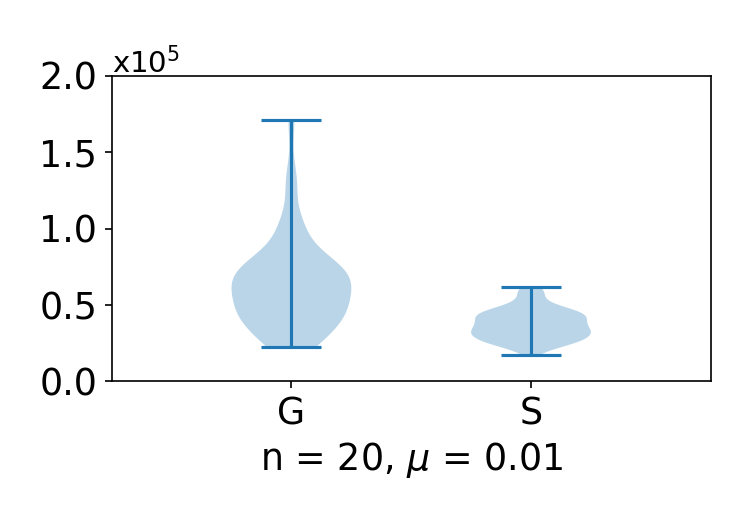}  
    \includegraphics[scale =  \SCALE]{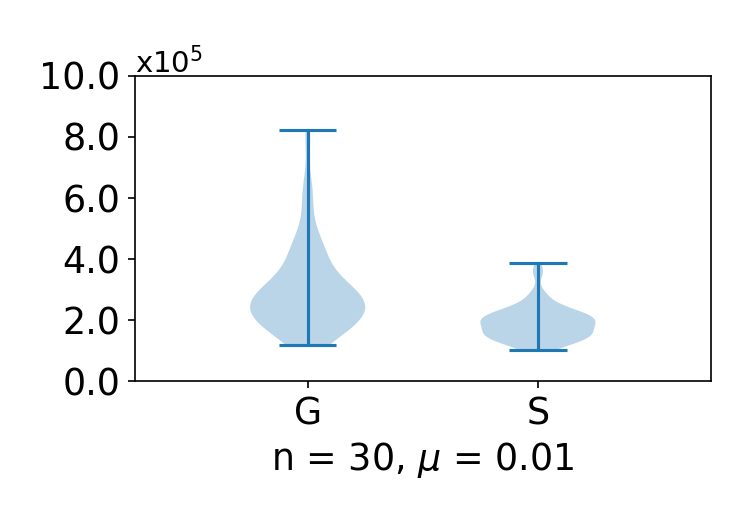} \\
    \includegraphics[scale =  \SCALE]{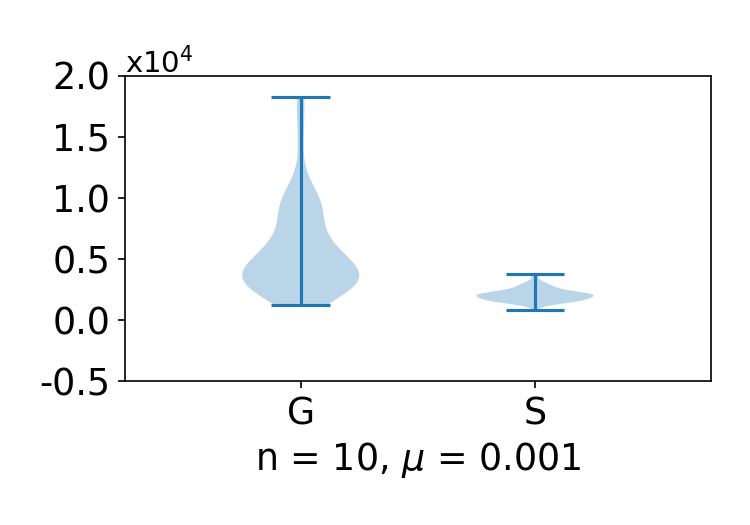}  
    \includegraphics[scale =  \SCALE]{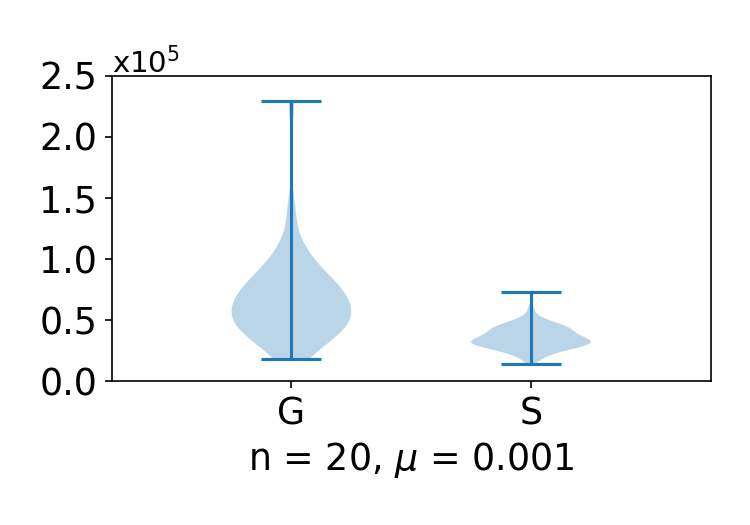}  
    \includegraphics[scale =  \SCALE]{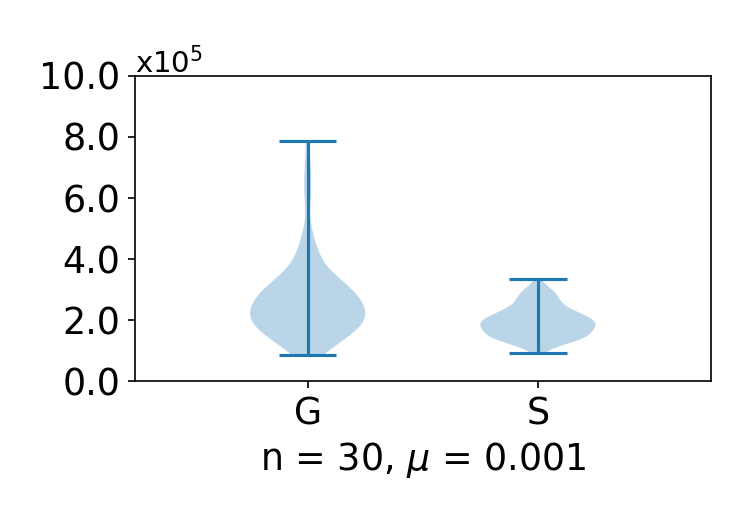} 
    \vspace*{-0.3cm}
    \caption{Plot of $ \frac{1}{m} \sum_{i=1}^m \left\| \wh{\grad} f\big|_p (v_i) \right\|^4 $ for Setting 1b.
    \label{fig:1b}}
\end{figure}

\begin{figure}[H]
    \centering
    \includegraphics[scale =  \SCALE]{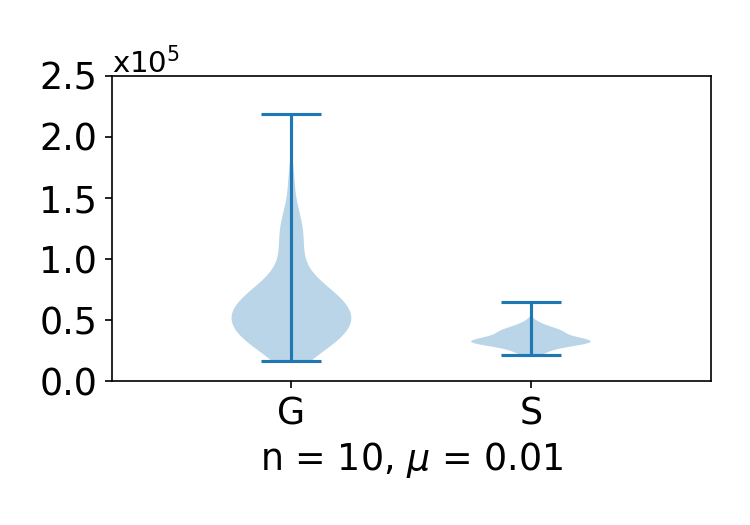}  
    \includegraphics[scale =  \SCALE]{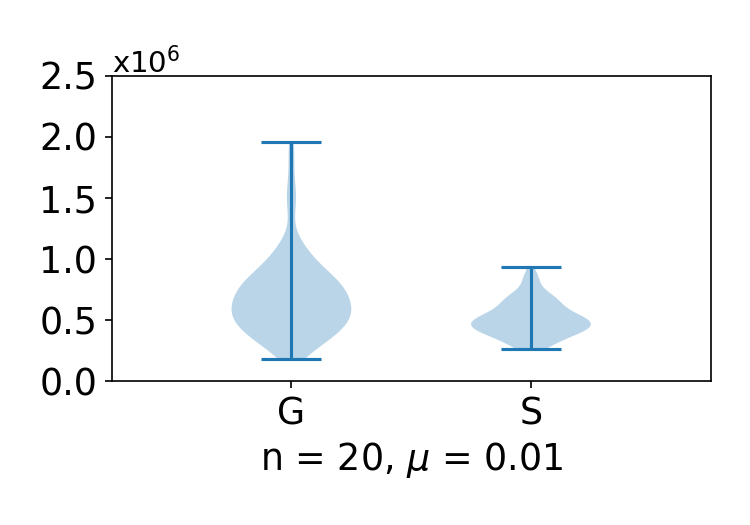}  
    \includegraphics[scale =  \SCALE]{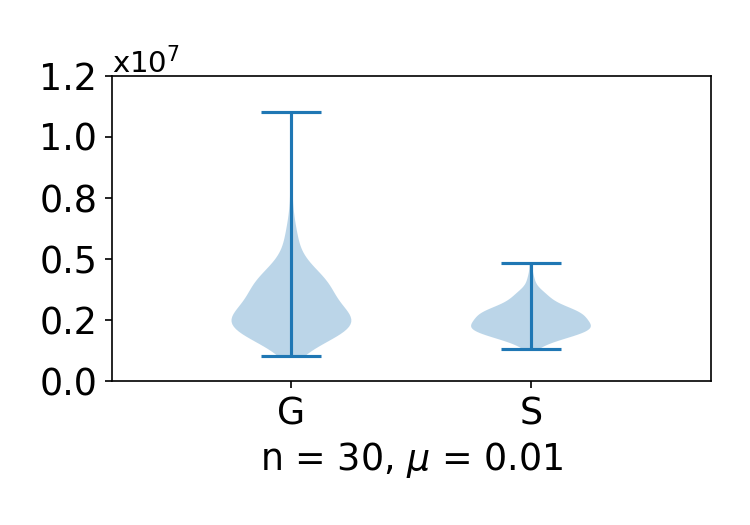} \\
    \includegraphics[scale =  \SCALE]{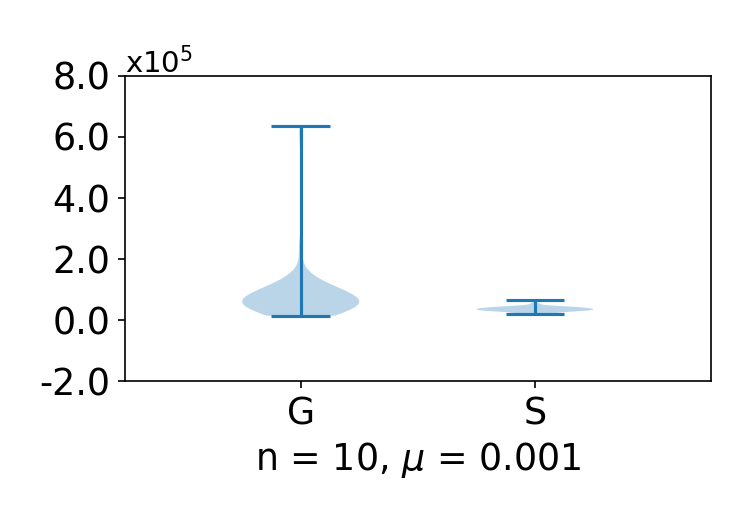}  
    \includegraphics[scale =  \SCALE]{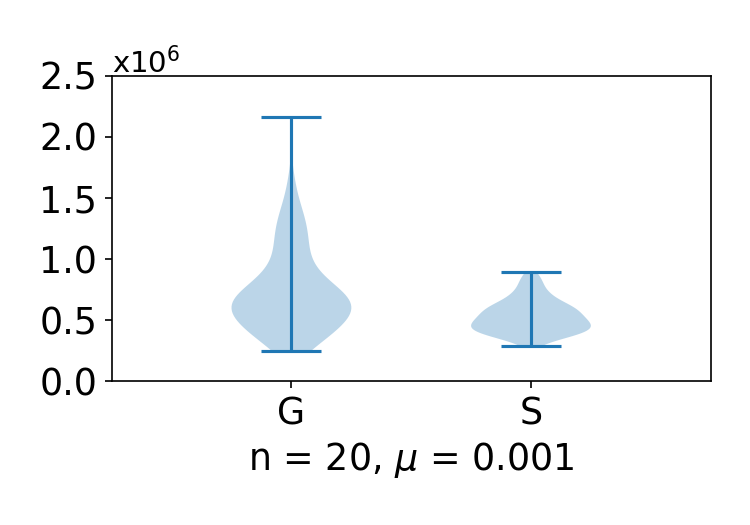}  
    \includegraphics[scale =  \SCALE]{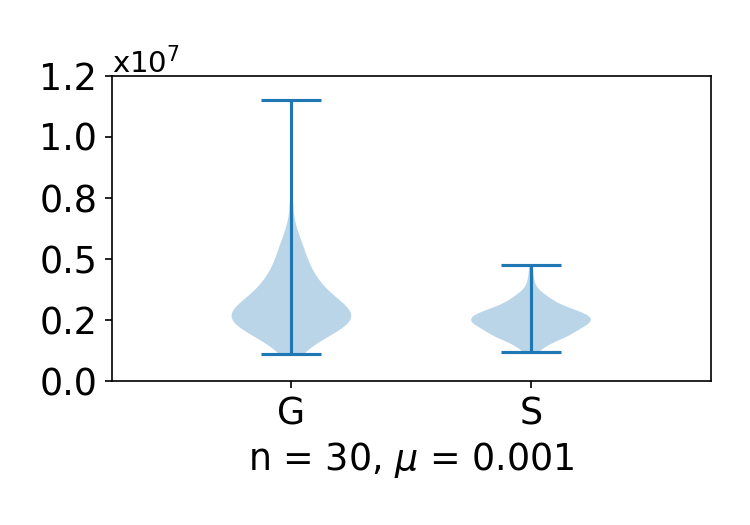} 
    \vspace*{-0.3cm}
    \caption{Plot of $ \frac{1}{m} \sum_{i=1}^m \left\| \wh{\grad} f\big|_p (v_i) \right\|^4 $ for Setting 2a.
    \label{fig:2a}}
\end{figure}

\begin{figure}[H]
    \centering
    \includegraphics[scale =  \SCALE]{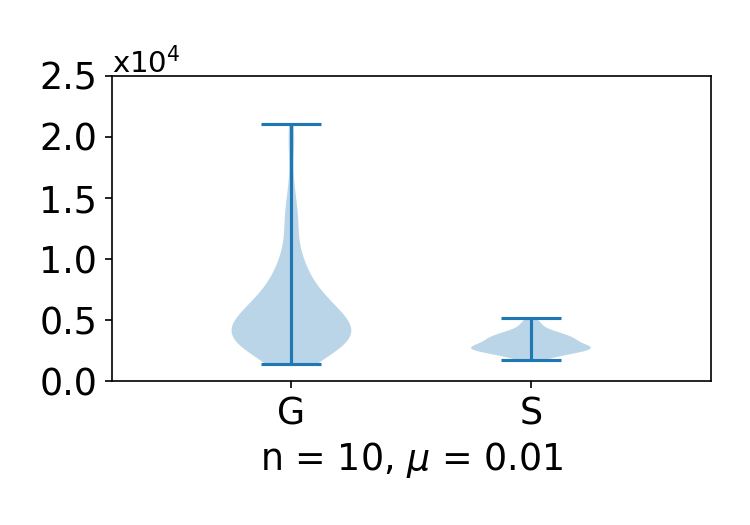}  
    \includegraphics[scale =  \SCALE]{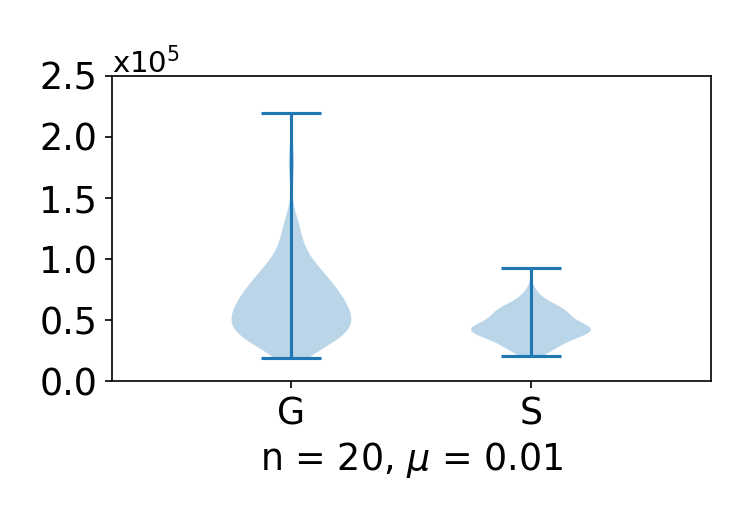}  
    \includegraphics[scale =  \SCALE]{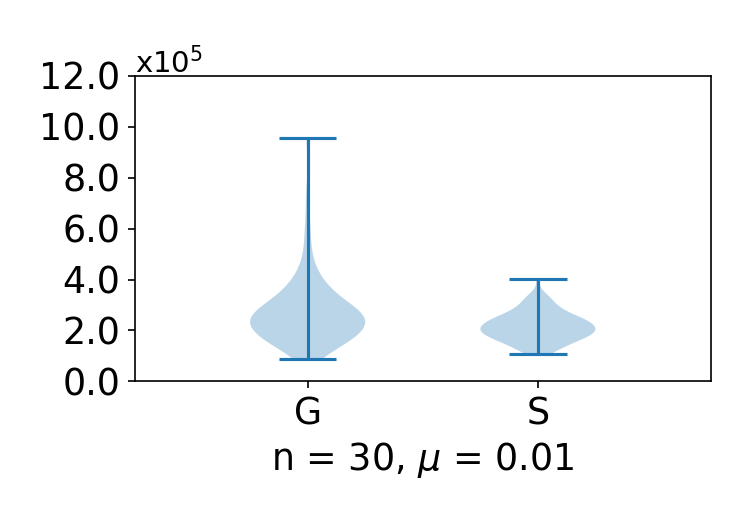} \\
    \includegraphics[scale =  \SCALE]{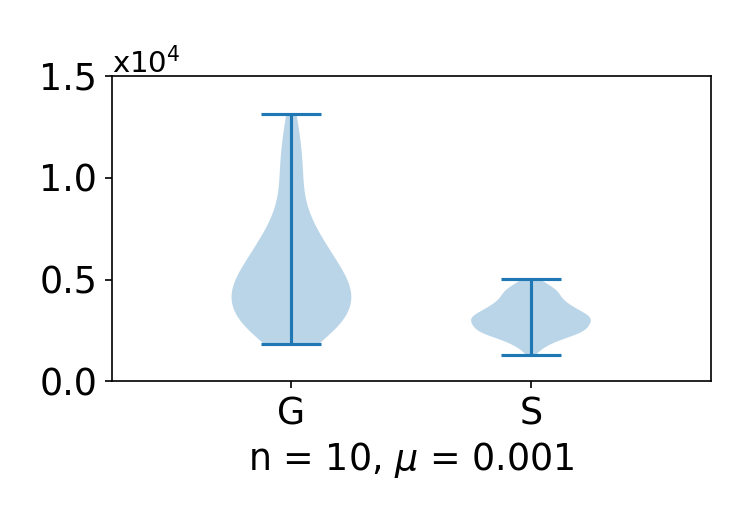}  
    \includegraphics[scale =  \SCALE]{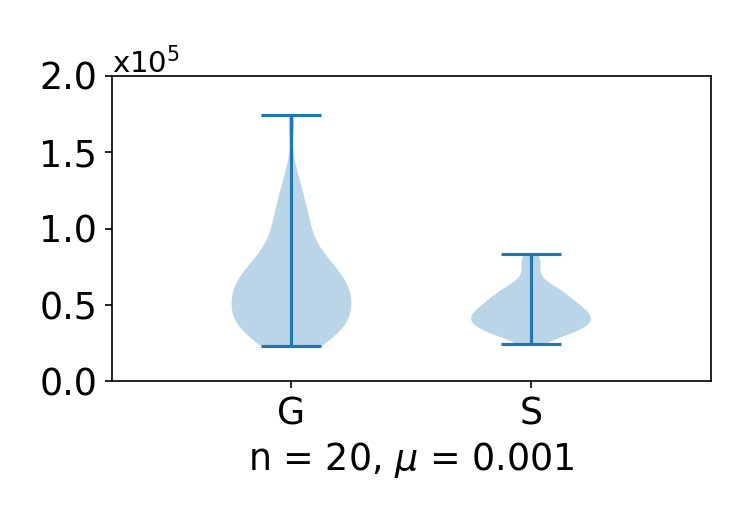}  
    \includegraphics[scale =  \SCALE]{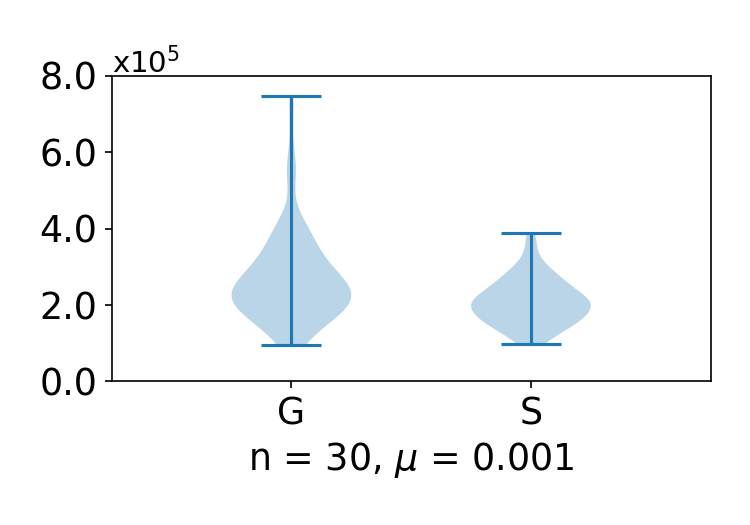} 
    \vspace*{-0.3cm} 
    \caption{Plot of $ \frac{1}{m} \sum_{i=1}^m \left\| \wh{\grad} f\big|_p (v_i) \right\|^4 $ for Setting 2b.
    \label{fig:2b} }
\end{figure}

\begin{figure}[H]
    \centering
    \includegraphics[scale =  \SCALE]{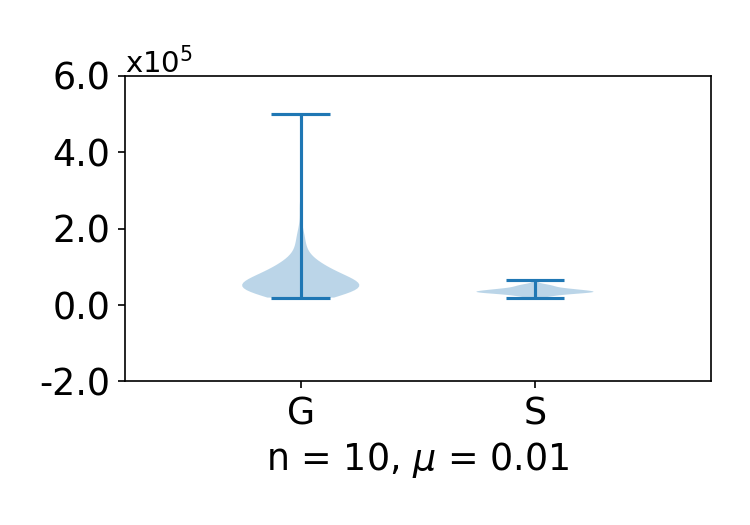}  
    \includegraphics[scale =  \SCALE]{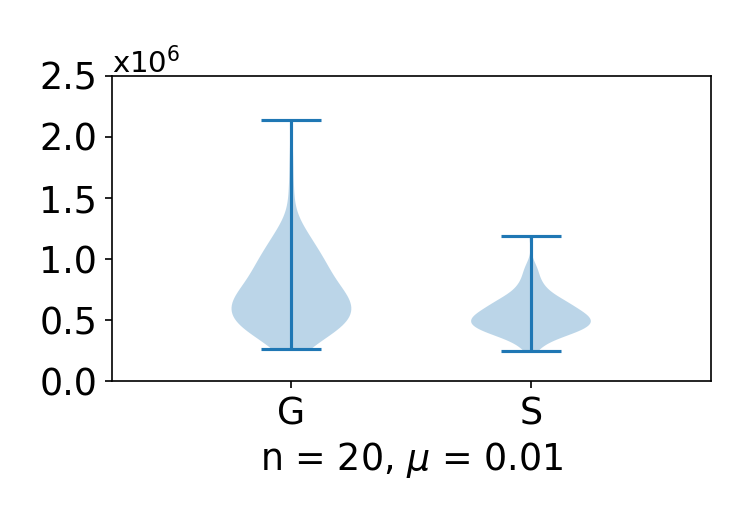}  
    \includegraphics[scale =  \SCALE]{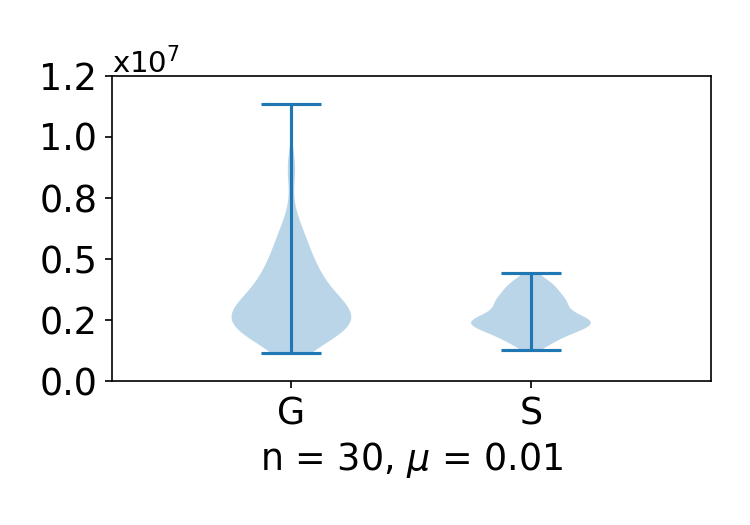} \\
    \includegraphics[scale =  \SCALE]{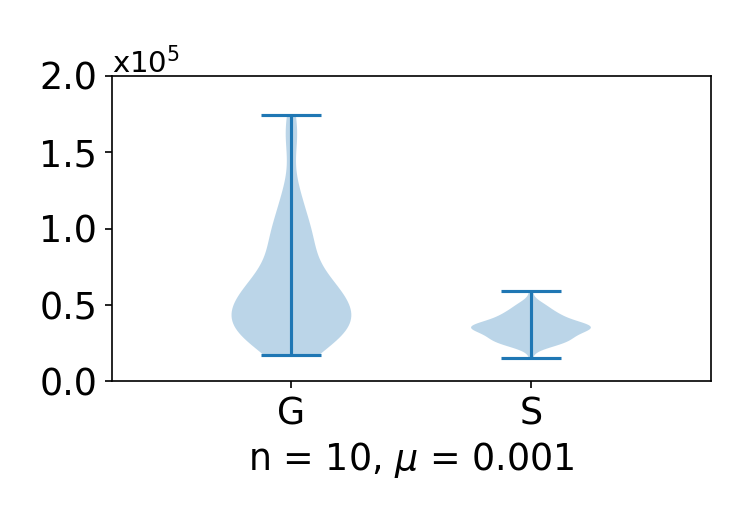}  
    \includegraphics[scale =  \SCALE]{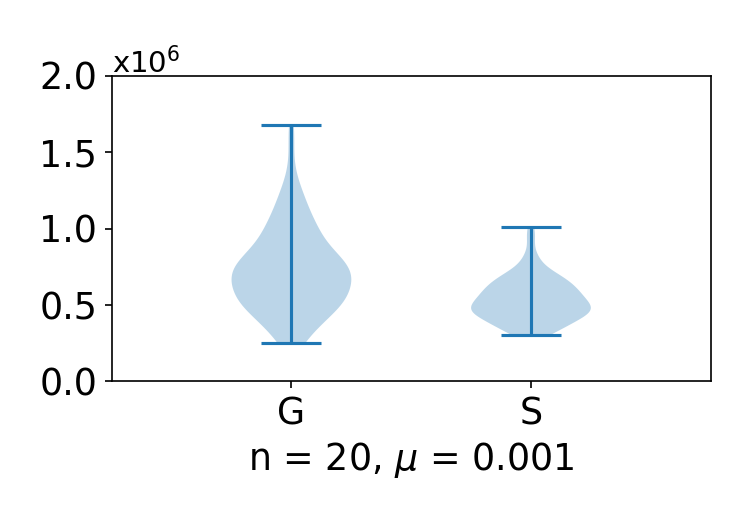}  
    \includegraphics[scale =  \SCALE]{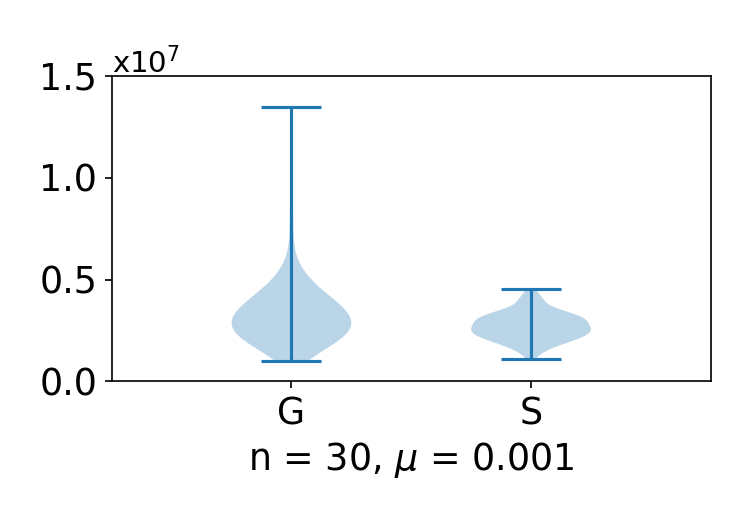}  
    \vspace*{-0.3cm} 
    \caption{Plot of $ \frac{1}{m} \sum_{i=1}^m \left\| \wh{\grad} f\big|_p (v_i) \right\|^4 $ for Setting 3a.
    \label{fig:3a} }
\end{figure}

\begin{figure}[H]
    \centering
    \includegraphics[scale =  \SCALE]{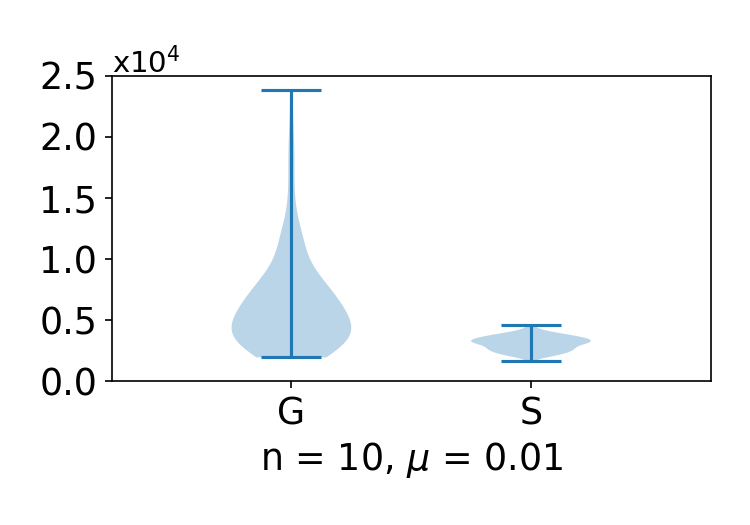} 
    \includegraphics[scale =  \SCALE]{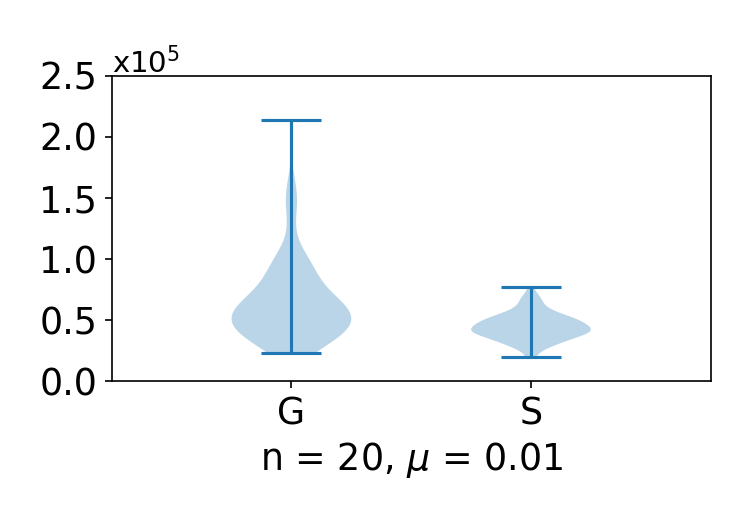} 
    \includegraphics[scale =  \SCALE]{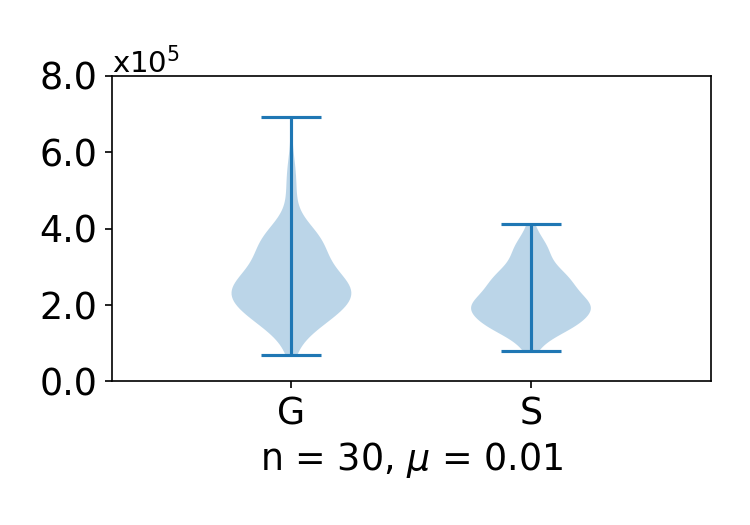} \\
    \includegraphics[scale =  \SCALE]{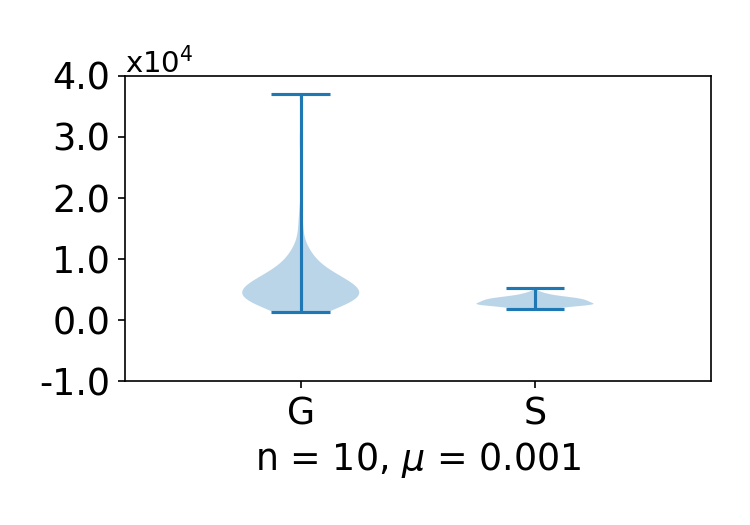} 
    \includegraphics[scale =  \SCALE]{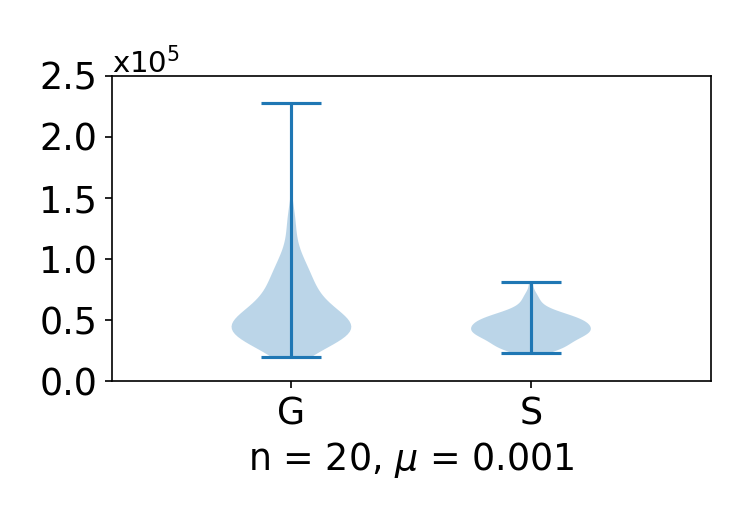} 
    \includegraphics[scale =  \SCALE]{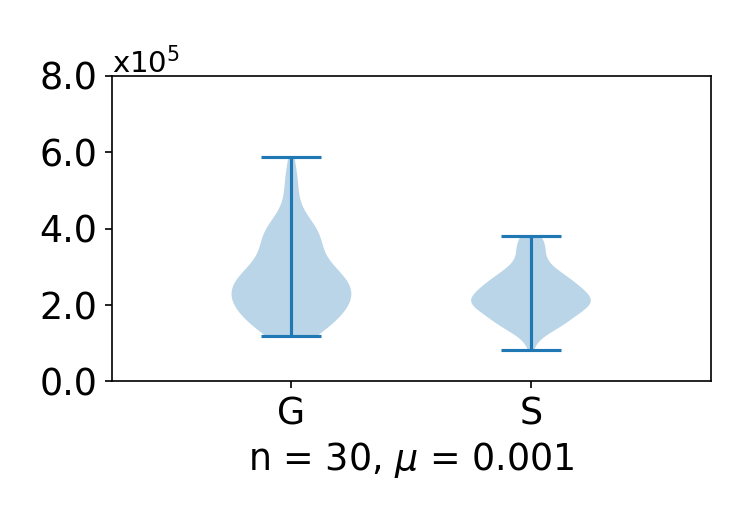} 
    \vspace*{-0.3cm} 
    \caption{Plot of $ \frac{1}{m} \sum_{i=1}^m \left\| \wh{\grad} f\big|_p (v_i) \right\|^4 $ for Setting 3b. 
    \label{fig:3b} }
\end{figure}


\begin{figure}[H]
    \centering
    \includegraphics[scale =  \SCALE]{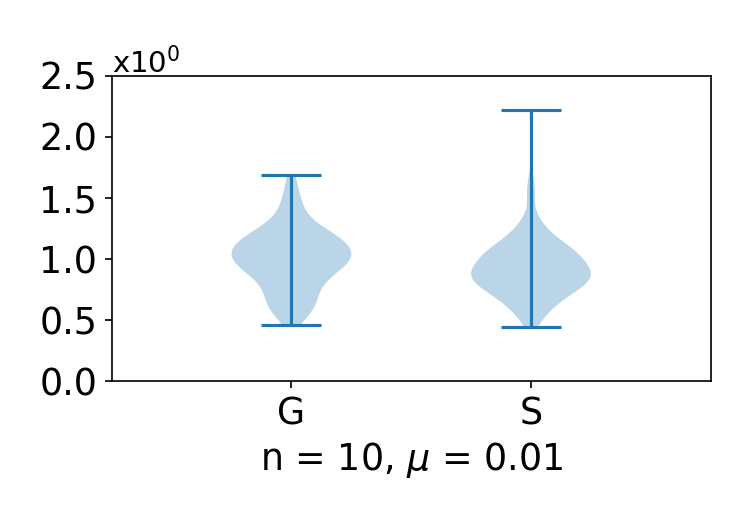}  
    \includegraphics[scale =  \SCALE]{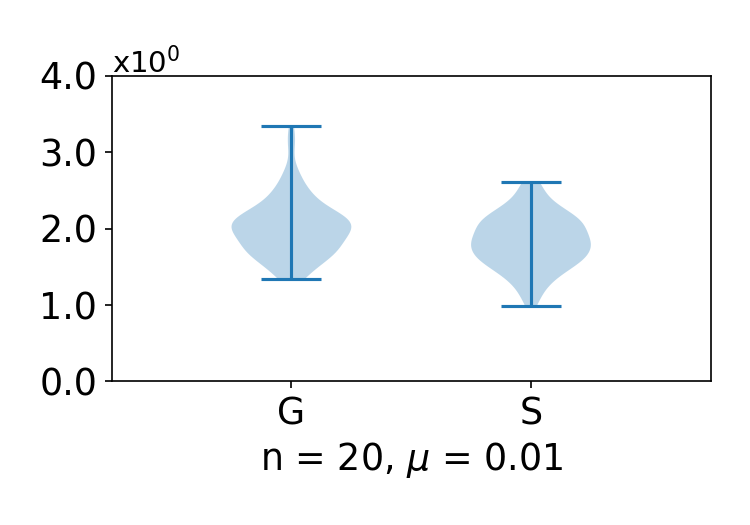}  
    \includegraphics[scale =  \SCALE]{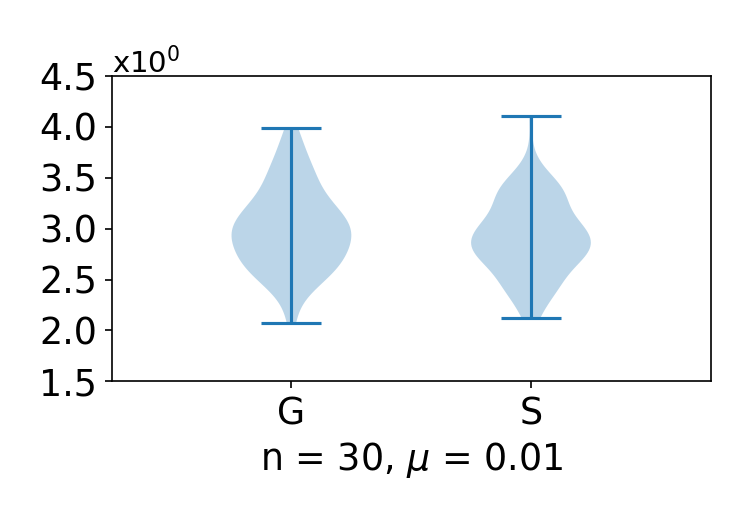} \\
    \includegraphics[scale =  \SCALE]{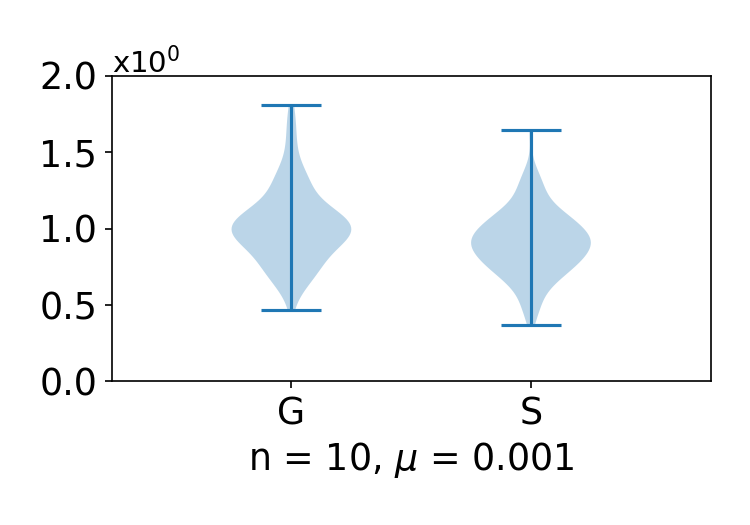}  
    \includegraphics[scale =  \SCALE]{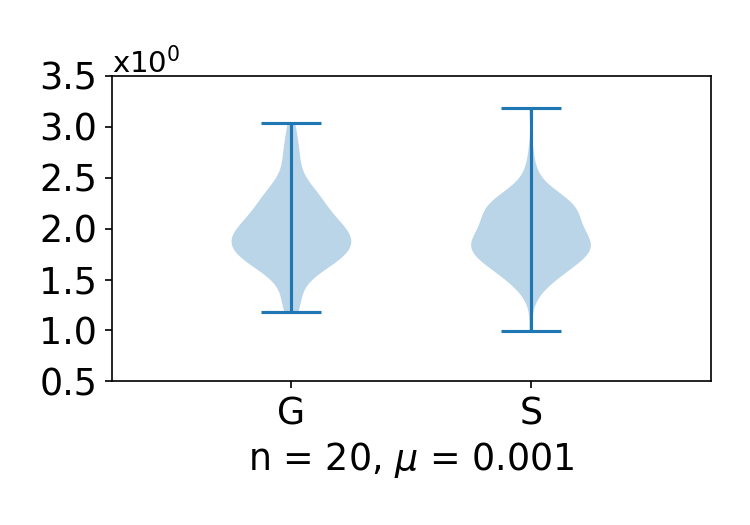}  
    \includegraphics[scale =  \SCALE]{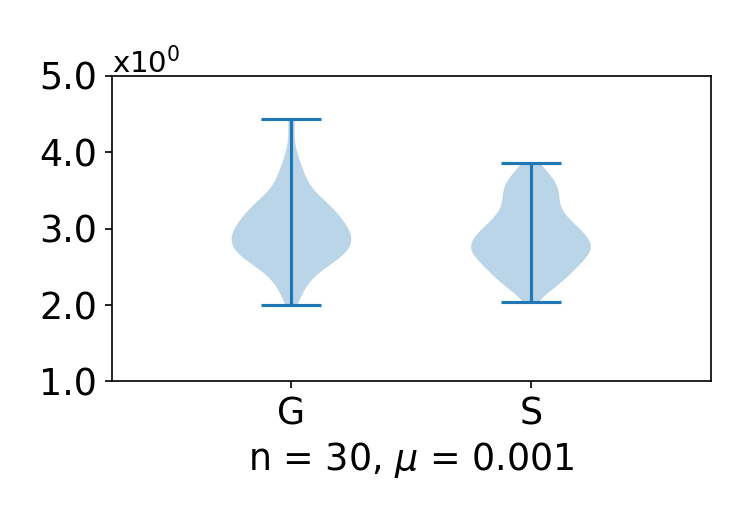} \vspace*{-0.3cm}
    \caption{Plot of $ \left\| \wh{\grad} f\big|_p (v_1, v_2, \cdots, v_m) - \grad f \big|_p \right\| $ for Setting 1a.
    \label{fig:error-1a} }
\end{figure}

\begin{figure}[H]
    \centering
    \includegraphics[scale =  \SCALE]{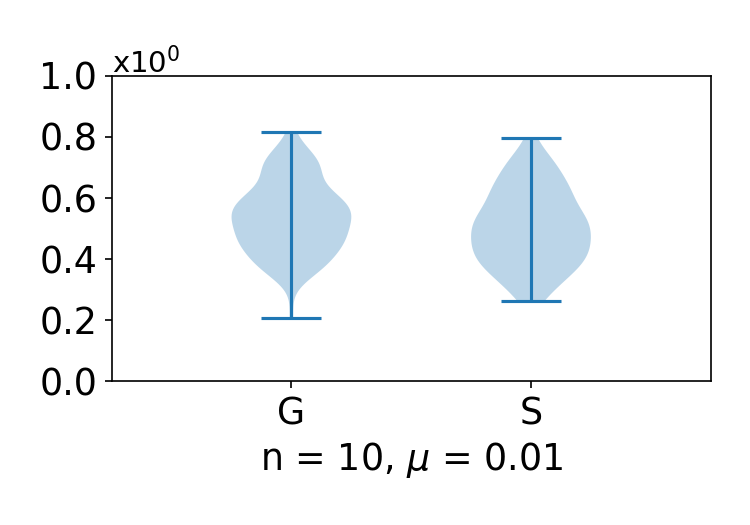}  
    \includegraphics[scale =  \SCALE]{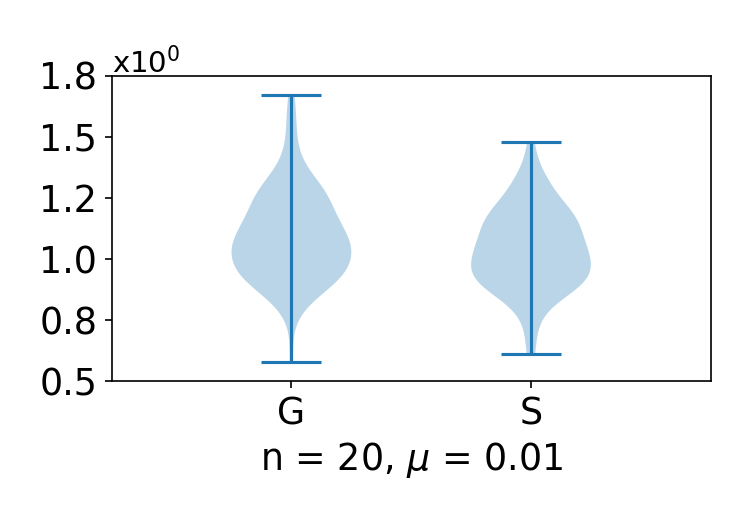}  
    \includegraphics[scale =  \SCALE]{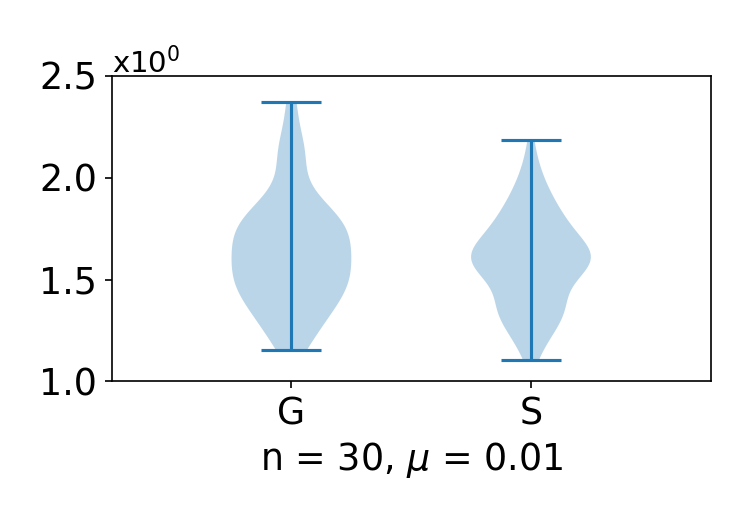} \\
    \includegraphics[scale =  \SCALE]{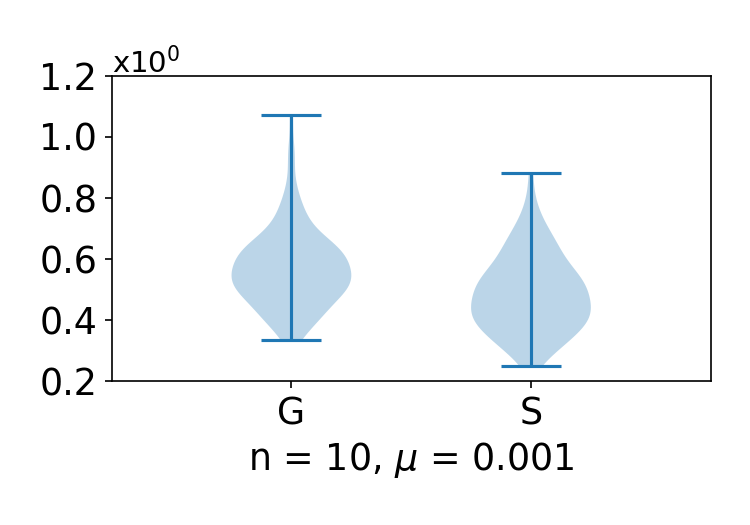}  
    \includegraphics[scale =  \SCALE]{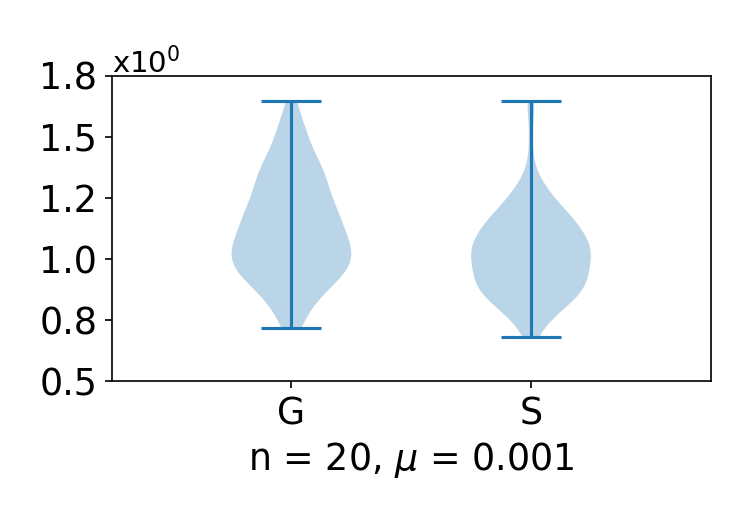}  
    \includegraphics[scale =  \SCALE]{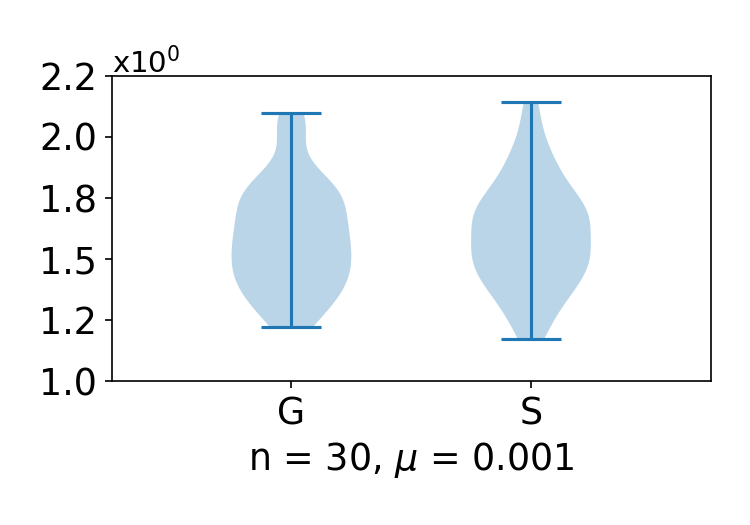} 
    \vspace*{-0.3cm}
    \caption{Plot of $ \left\| \wh{\grad} f\big|_p (v_1, v_2, \cdots, v_m) - \grad f \big|_p \right\| $ for Setting 1b.
    \label{fig:error-1b}}
\end{figure}

\begin{figure}[H]
    \centering
    \includegraphics[scale =  \SCALE]{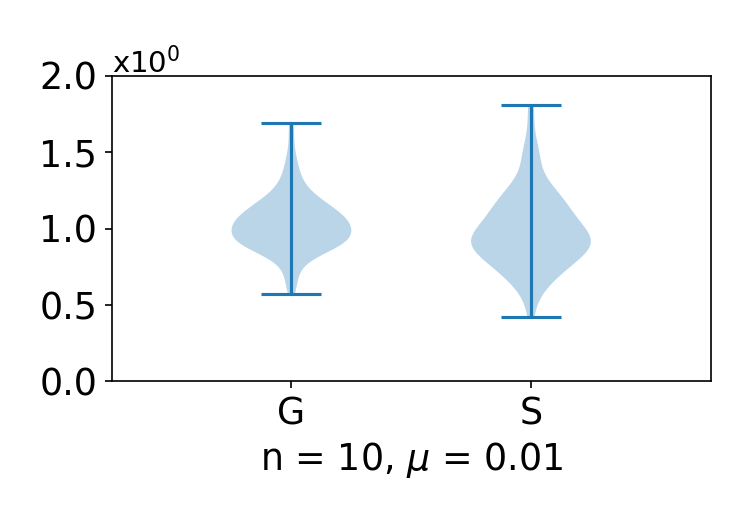}  
    \includegraphics[scale =  \SCALE]{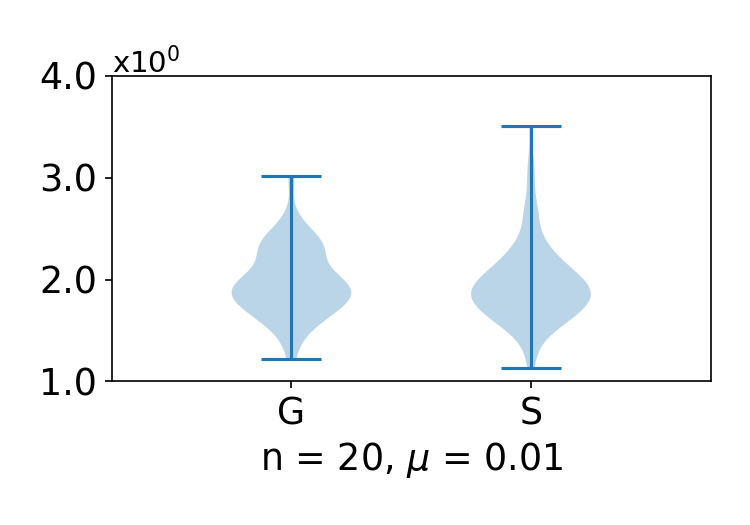}  
    \includegraphics[scale =  \SCALE]{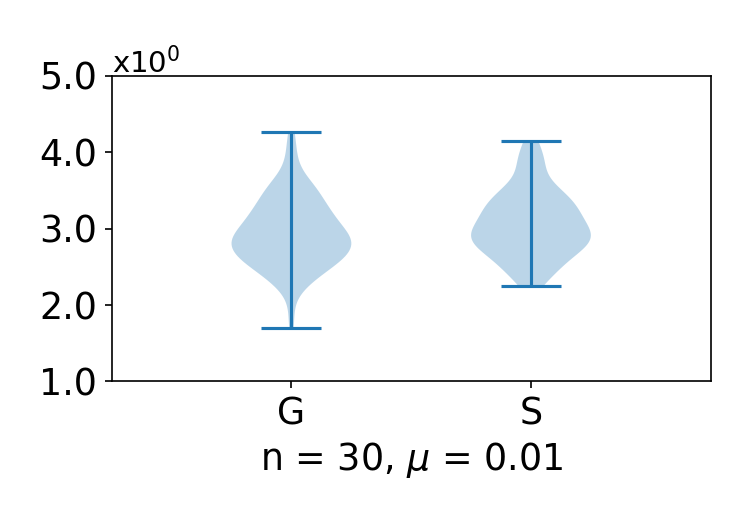} \\
    \includegraphics[scale =  \SCALE]{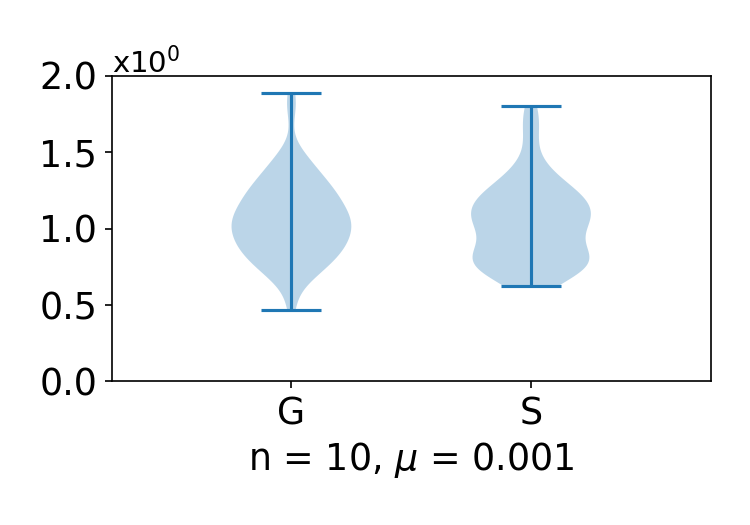}  
    \includegraphics[scale =  \SCALE]{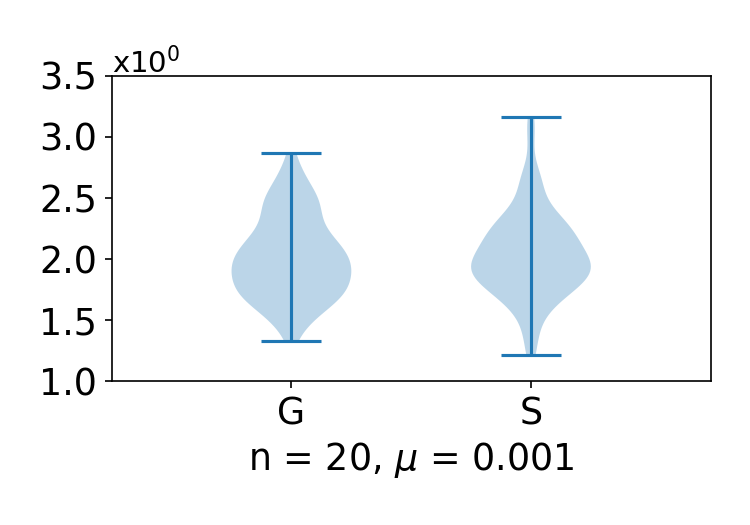}  
    \includegraphics[scale =  \SCALE]{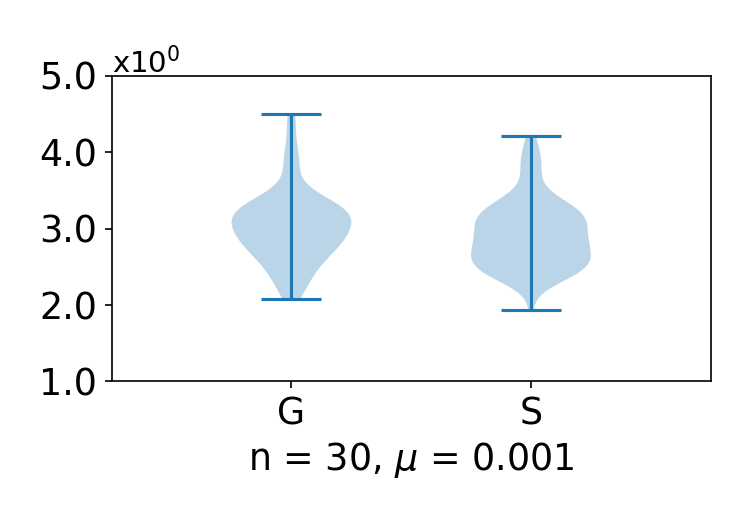} 
    \vspace*{-0.3cm}
    \caption{Plot of $ \left\| \wh{\grad} f\big|_p (v_1, v_2, \cdots, v_m) - \grad f \big|_p \right\| $ for Setting 2a.
    \label{fig:error-2a}}
\end{figure}

\begin{figure}[H]
    \centering
    \includegraphics[scale =  \SCALE]{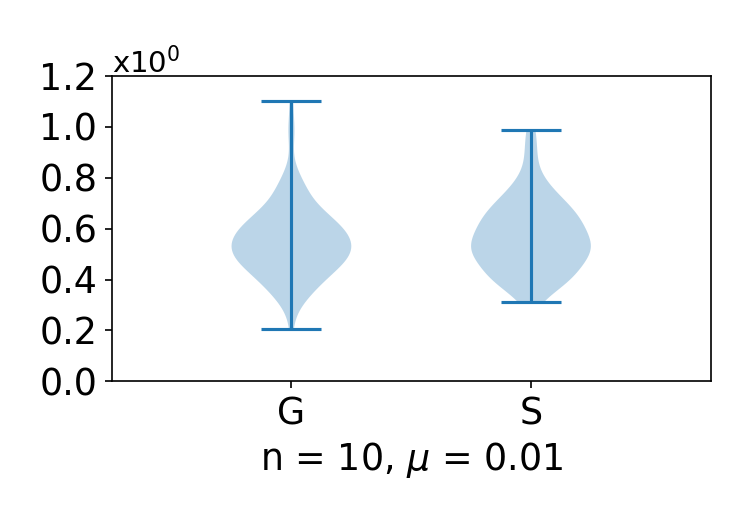}  
    \includegraphics[scale =  \SCALE]{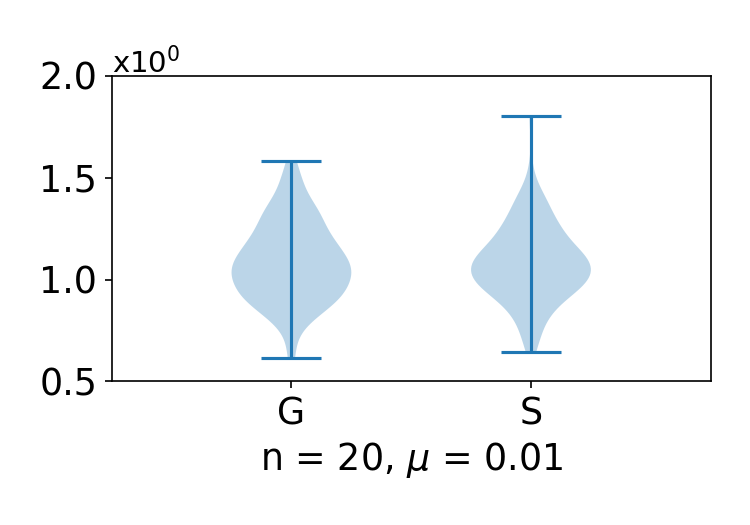}  
    \includegraphics[scale =  \SCALE]{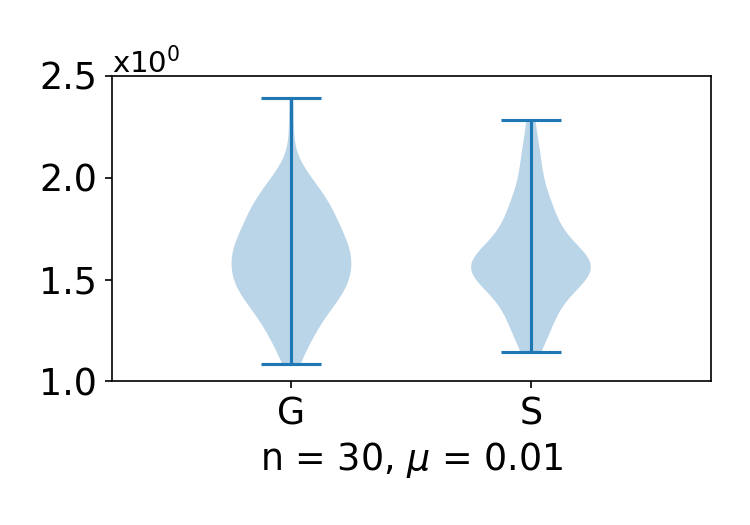} \\
    \includegraphics[scale =  \SCALE]{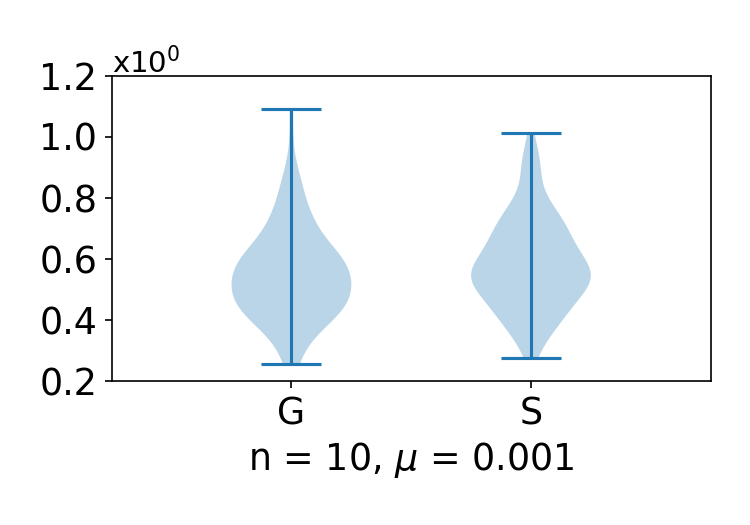}  
    \includegraphics[scale =  \SCALE]{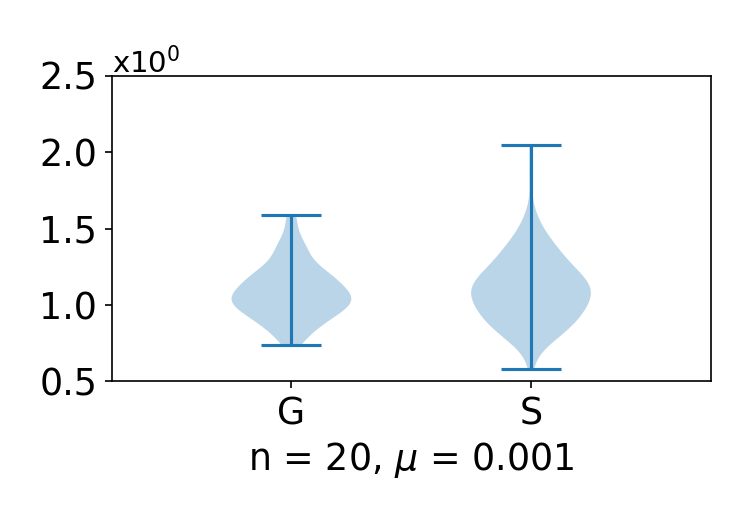}  
    \includegraphics[scale =  \SCALE]{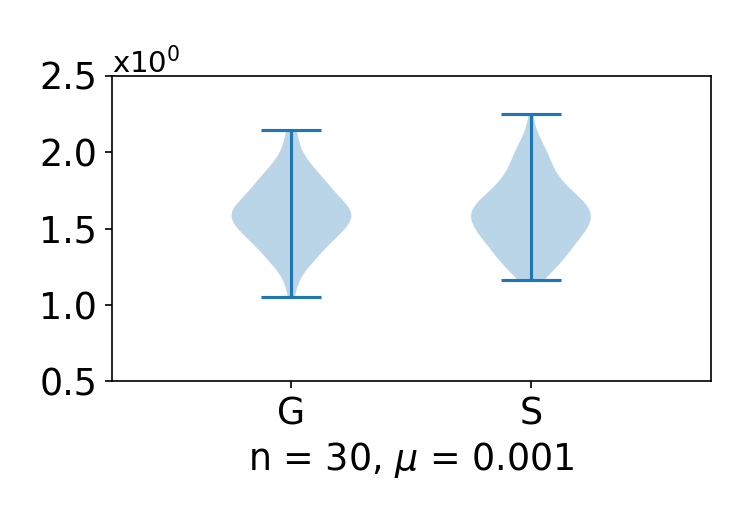} 
    \vspace*{-0.3cm} 
    \caption{Plot of $ \left\| \wh{\grad} f\big|_p (v_1, v_2, \cdots, v_m) - \grad f \big|_p \right\| $ for Setting 2b.
    \label{fig:error-2b} } 
\end{figure}

\begin{figure}[H]
    \centering
    \includegraphics[scale =  \SCALE]{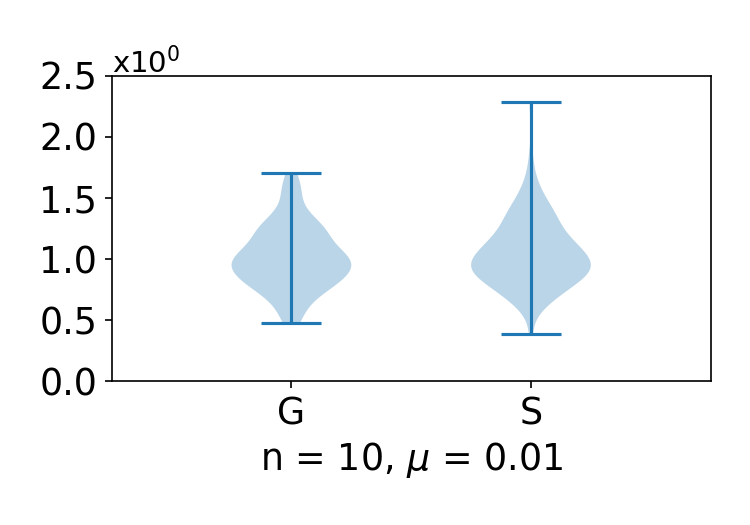}  
    \includegraphics[scale =  \SCALE]{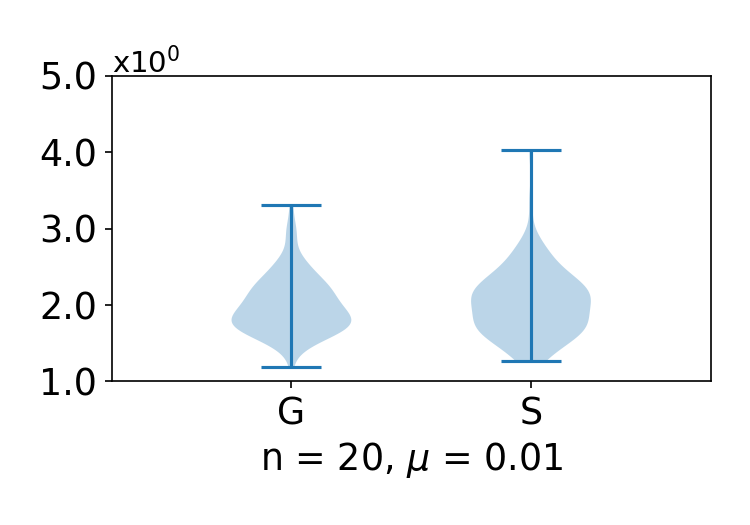}  
    \includegraphics[scale =  \SCALE]{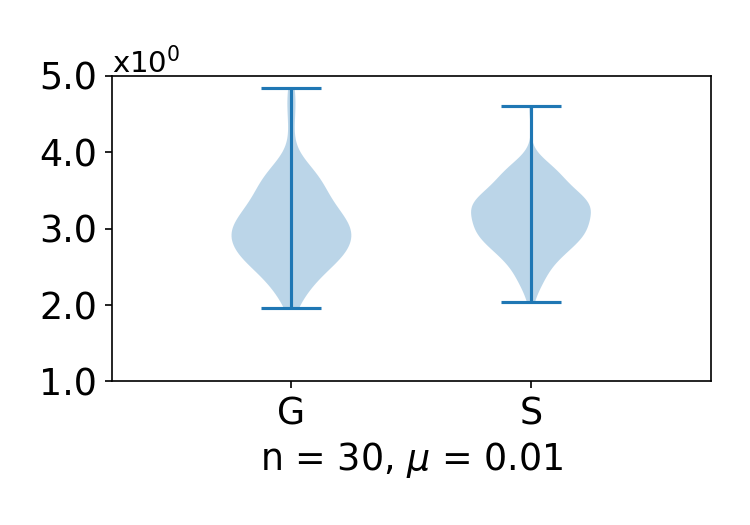} \\
    \includegraphics[scale =  \SCALE]{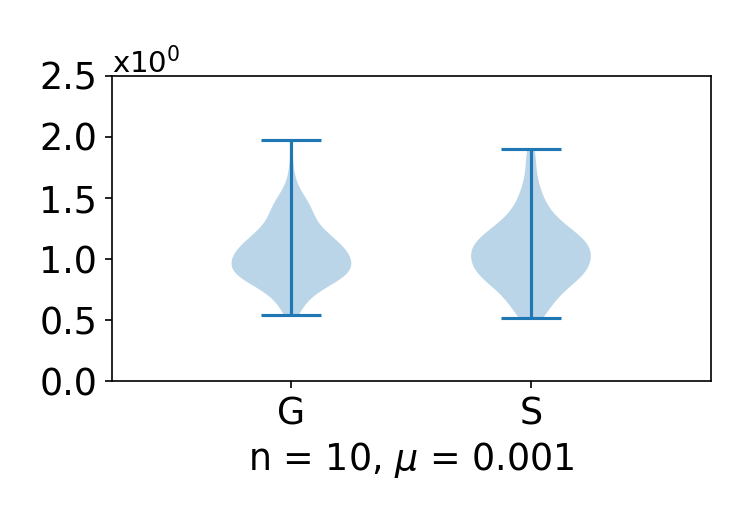}  
    \includegraphics[scale =  \SCALE]{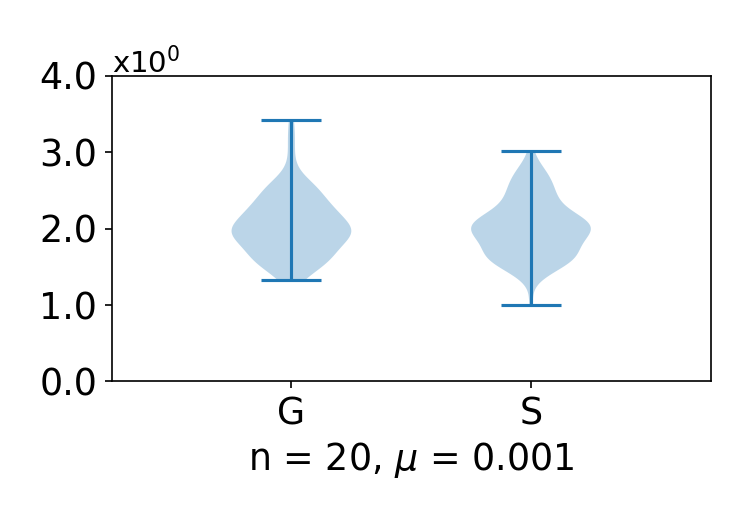}  
    \includegraphics[scale =  \SCALE]{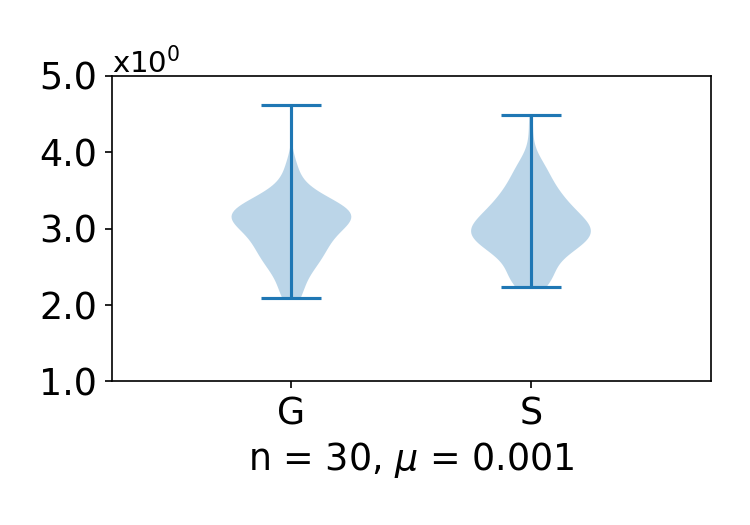}  
    \vspace*{-0.3cm} 
    \caption{Plot of $ \left\| \wh{\grad} f\big|_p (v_1, v_2, \cdots, v_m) - \grad f \big|_p \right\| $ for Setting 3a.
    \label{fig:error-3a} }
\end{figure}

\begin{figure}[H]
    \centering
    \includegraphics[scale =  \SCALE]{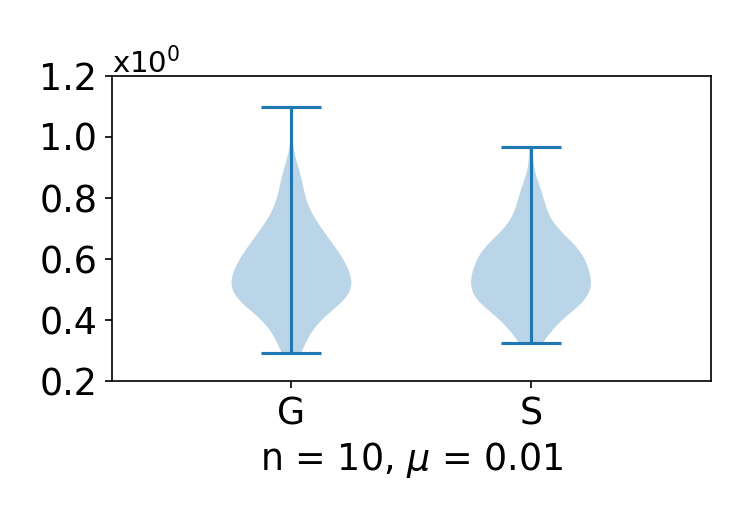} 
    \includegraphics[scale =  \SCALE]{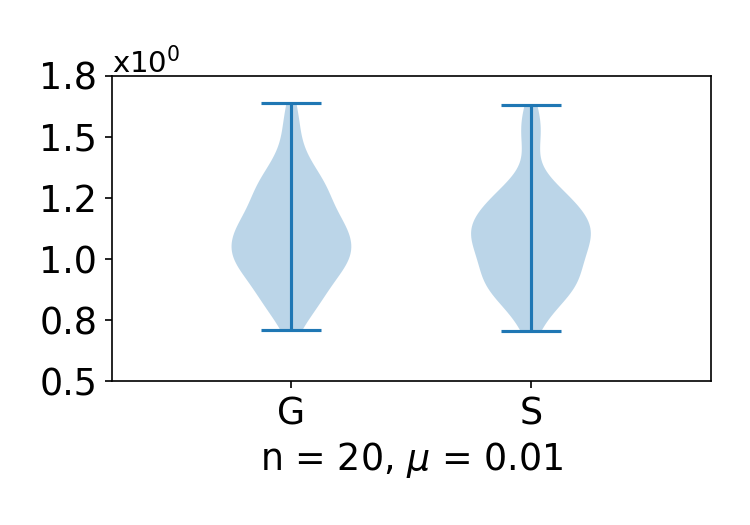} 
    \includegraphics[scale =  \SCALE]{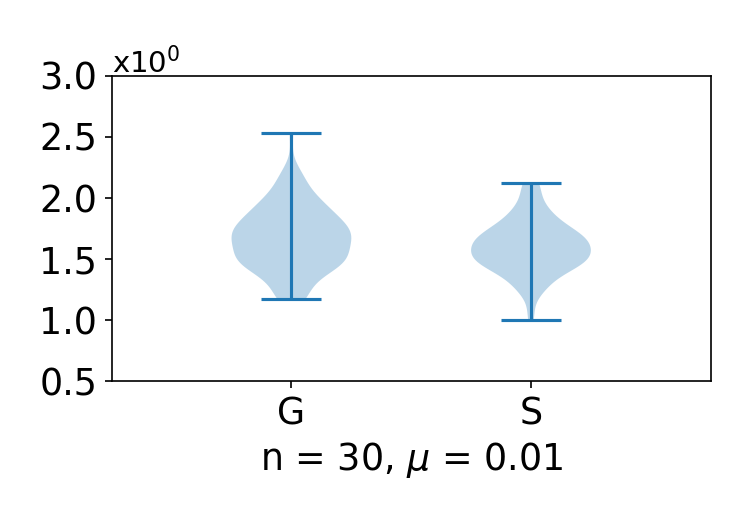} \\
    \includegraphics[scale =  \SCALE]{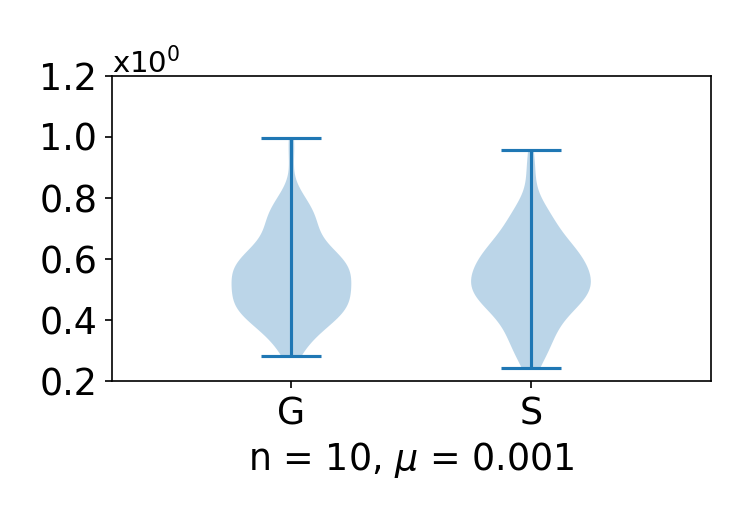} 
    \includegraphics[scale =  \SCALE]{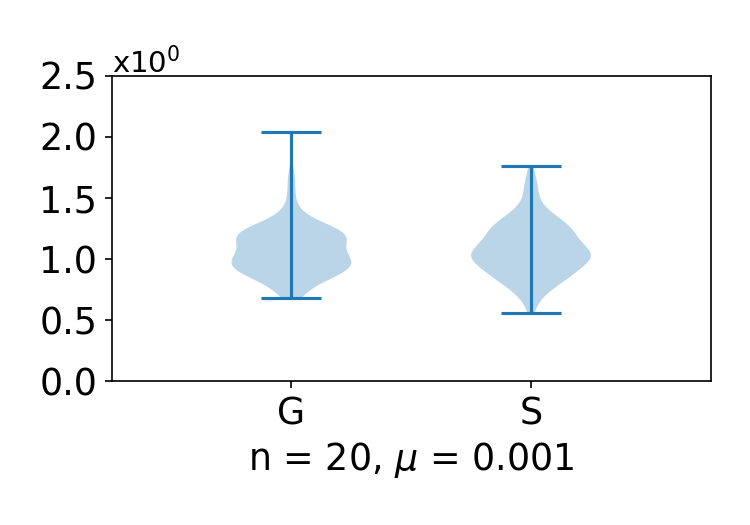} 
    \includegraphics[scale =  \SCALE]{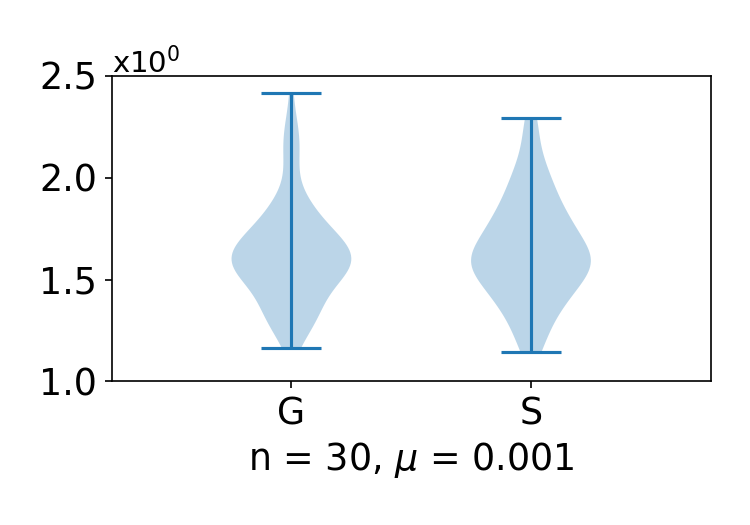} 
    \vspace*{-0.3cm} 
    \caption{Plot of $ \left\| \wh{\grad} f\big|_p (v_1, v_2, \cdots, v_m) - \grad f \big|_p \right\| $ for Setting 3b.
    \label{fig:error-3b} }
\end{figure}


\section{Application to Online Convex Optimization}


In online learning with bandit feedback, each time $t$, the learner chooses points $x_t$ and $z_t$, and the environment picks a convex loss function $f_t$. The learner suffers and observes loss $f_t (x_t)$ and observes the function value $f_t ( z_t )$. The learner then picks $x_{t+1}$ based on historical observations and the learning process moves on. 
Online learning agents, unlike their offline learning (or batch learning) counterpart, do not have knowledge about $f_t$ when they choose $x_t$. Such models capture real-world scenarios such as learning with streaming data. In the machine learning community, such problem settings are often referred to as the \textit{Online Convex Optimization} problem with bandit-type feedback.

\subsection{Online Convex Optimization over Hadamard Manifolds: an Algorithm} 

We consider an online convex optimization problem (with bandit-type feedback) over a negatively curved manifold. Let there be a Hadamard manifold $\M$ and a closed geodesic ball $\D \subseteq \M$. Each time $t$, the agent and the environment executes the above learning protocol, with $f_t$ defined over $\M$ and $x_t \in \D$. The goal of the agent is to minimize the $T$-step hindsight regret: $ Reg (T) := \sum_{t=1}^T f_t (x_t) - f_t (x^*) $, where $x^* \in \arg\min_{x \in \D } \sum_{t=1}^T f_t (x)$. 

With the approximation guarantee in Theorem \ref{thm:grad-error}, one can solve this online learning problem. At each $t$, we can approximate gradient at $ x_t $ using the function value $f_t \( z_t \) $, and perform an online gradient descent step. In particular, we can generalize the online convex optimization algorithm (with bandit-type feedback) \citep{flaxman2005online} to Hadamard manifolds. 

By Theorem \ref{thm:grad-error}, we have 
\begin{align}
    \left\| \grad f \big|_p - \E_{v \sim \S_p} \[ \frac{n}{\mu} \( f (\Exp_p (\mu v) ) - f (p) \) v \] \right\| 
    \le 
    \frac{ L_1 n \mu }{ 2 }, \qquad \forall p \in \M \label{eq:bias-2-pt}
\end{align} 

Thus the gradient of $f_t$ at $x_t$ can be estimated by
\begin{align*}
    \grad f_t \big|_{x_t} \approx \E_{v_t \sim \S_{x_t}} \[ \frac{n}{\mu } \( f_t \( \Exp_{x_t} \( \mu v_t \)  \) - f_t (x_t) \) v_t \], 
\end{align*}
where $v_t$ is uniformly sampled from $\S_{x_t}$. This estimator generalizes the two-evaluation estimator in the Euclidean case \citep{duchi2015optimal}. As shown in (\ref{eq:bias-2-pt}), the bias of this estimator is also bounded by $ \frac{ L_1 n \mu }{ 2 } $. 

With this gradient estimator, one can perform estimated gradient descent by 
\begin{align*} 
    x_{t+\frac{1}{2}} \leftarrow \Exp_{x_t} \( - \eta \frac{n}{\mu } \( f_t \( \Exp_{x_t} \( \mu v_t \)  \) - f_t (x_t) \) v_t \), 
\end{align*}
where $v_t$ is uniformly sampled from $ \S_{x_t} $, and $\eta$ is the learning rate. 

To ensure that the algorithm stays in the geodesic ball $\D$, we need to project $x_{t+\frac{1}{2}}$ back to a subset of $\D$, which is 
\begin{align*}
    (1 - \beta) \D := B (c (\D), (1-\beta) r (\D) ), \qquad \text{ for some } \beta \in (0, 1) ,  
\end{align*}
where $c (\D)$ and $r (\D)$ are the center and radius of the geodesic ball $\D$. 
Since the set $(1 - \beta) \D$ is a geodesic ball in a Hadamard manifold, the projection onto it is well-defined. Thus we can project $ x_{t+\frac{1}{2}} $ onto $ (1-\beta) \D $ to get $x_{t+1}$, or let $ x_{t+1} = \Proj_{(1 - \beta) \D} \(x_{t + \frac{1}{2}} \) $. Now one iteration finishes and the learning process moves on to the next round. This learning algorithm is summarized in Algorithm \ref{alg}.

\begin{algorithm}[h]
    \caption{}  
    \label{alg}
    \begin{algorithmic} 
        \STATE \textbf{Input:} Algorithm parameter: $\mu $, $\eta$, $\beta$; time horizon: $T$;
        \STATE \textbf{Initialization:} Arbitrarily pick $x_1 \in \D$. 
        \FOR {$t = 1, 2,  \dots ,T$} 
        	\STATE Environment picks loss function $ f_t$. The agent suffers and observes loss $f_t (x_t)$. 
            \STATE Sample $v_t \sim \S_{x_t} $. 
        	\STATE Observe function value $ f_t (z_t)$, where 
        	$z_t = \Exp_{x_t} ( \mu v_t ) $, and compute $ w_t = \frac{ n }{ \mu } \( f_t (z_t) - f_t (x_t) \) v_t $. 
        	/* $w_t$ is the gradient estimate. */ 
        	\STATE Perform gradient update 
        	$x_{t+\frac{1}{2}} = \Exp_{x_t} ( -\eta w_t )$. 
        	\STATE Project the point to set $(1 - \beta) \D$, $x_{t+1} = \Proj_{(1 - \beta) \D} \(x_{t + \frac{1}{2}} \) $. 
        	
        	/* In practice, the projection can be approximated by a point in $(1 - \beta) \D$ that is close to $x_{t+\frac{1}{2}}$. */
        \ENDFOR 
    \end{algorithmic} 
\end{algorithm}

\subsection{Online Convex Optimization over Hadamard Manifolds: Analysis} 

\begin{theorem}
\label{thm:rate}
    If, for all $t \in [T]$, the loss function $ f_t$ satisfies that is geodesically $\mu$-smooth, 
    then the regret of Algorithm \ref{alg} satisfies:
    \begin{align}
        \E \[ Reg (T) \] \le \frac{D^2}{2 \eta} +
        \frac{\zeta (\kappa, D) \eta n^2 G^2 T }{2\mu^2 } + \( \frac{D L_1 n \mu }{2}  + 2 \beta D \) T, 
    \end{align}
    where $D$ is the diameter of $\D$, and $G := \max_{p \in \D} f (p)$.  
\end{theorem} 

In particular, when $\eta = \Theta \( T^{-3/4} \)$, $\mu = \Theta \( T^{-1/4} \)$ and $\beta = \Theta \( T^{-1/4} \)$, the regret satisfies $ \E \[ Reg (T) \] \le \mathcal{O} \( T^{3/4} \) $. This sublinear regret rate implies that, when $f_1, f_2, \cdots, f_T$ are identical, the algorithm converges to the optimal point at a rate of order $\mathcal{O} \( T^{-1/4} \)$. We now start the proof of Theorem \ref{thm:rate}, by presenting an existing result by Zhang and Sra. 


\begin{lemma}[Corollary 8, \cite{zhang2016first}] 
    \label{lem:zhang}
    Let $\M$ be a Hadamard manifold, 
    with sectional curvature $ \mathcal{K}_p (u,v) \ge \kappa > -\infty$ for all $p \in \M$ and $u,v \in T_p \M$. Then for any $x, x_s$, and $x_{s+1} = \Exp_{x_s} (-\eta v)$ for some $v \in T_{x_s} \M$, then it holds that 
    \begin{align}
        \< - v , \Exp_{x_s}^{-1} (x) \>_{x_s} \le& \frac{1}{2 \eta } \( d^2 (x_s, x) - d^2 (x_{s+1}, x) \) + \frac{\zeta (\kappa, d(x_s, x)) \eta }{2} \| v \|^2, \nonumber
    \end{align}
    where $ \zeta (\kappa, D) = \frac{ \sqrt{|\kappa|} D }{ \tanh (\sqrt{|\kappa|} D ) } $. 
\end{lemma}

\begin{proof}[Proof of Theorem \ref{thm:rate}]
    By convexity of $f_t$, we have 
    $$ f_t(x_{t}) - f_t(x^*) \le  \< - \grad f_t\big|_{x_{t} }, \Exp_{ x_{t} }^{-1} (x^*) \> .$$ 
    Together with the Lemma of Zhang and Sra, we get
    \begin{align*}
        &f_t ( x_{t } ) - f_t (x^*) \\
        \le& \< - \grad f_t\big|_{x_{t} }, \Exp_{x_{t } }^{-1} (x^*) \> \\
        \le& \< - w_t, \Exp_{x_{t } }^{-1} (x^*) \> + \< w_t - \grad f_t\big|_{x_{t} }, \Exp_{x_{t } }^{-1} (x^*) \> \\
        \le& 
        \frac{1}{2 \eta } 
        \( d^2 (x_{t}, x^*) - d^2 
        (x_{t+\frac{1}{2}}, x^*) \) + 
        \frac{\zeta (\kappa, D) \eta }{2} \| w_t \|^2 \\
        &+ \< w_t - \grad f_t \big|_{x_{t} }, \Exp_{x_{t } }^{-1} (x^*) \>,
    \end{align*}
    where the last line uses Lemma \ref{lem:zhang}. Taking expectation (over $v_t \sim \S_{x_t} $) on both sides gives 
    \begin{align}
        &\E \[ f_t ( x_{t } ) \] - \E \[ f_t (x^*) \] \nonumber \\
        \le& 
        \frac{1}{2 \eta } 
        \E \[ d^2 (x_{t}, x^*) - d^2 
        (x_{t+\frac{1}{2}}, x^*) \] \hspace{-2pt} + \hspace{-2pt}
        \E \[ \frac{\zeta (\kappa, D) \eta }{2} \| w_t \|^2 \] \nonumber \\
        &+ \< \E \[ w_t\] - \grad f_t \big|_{x_{t} }, \Exp_{x_{t } }^{-1} (x^*) \> \nonumber \\ 
        \le& 
        \frac{ \E \hspace{-2pt} \[ d^2 ( \hspace{-1pt} x_{t}, \hspace{-1pt} x^* \hspace{-1pt} )  \hspace{-2pt} - \hspace{-2pt} d^2 ( \hspace{-1pt} x_{t+\frac{1}{2}}, \hspace{-1pt} x^* \hspace{-1pt} ) \hspace{-1.5pt} \] }{ 2 \eta } \hspace{-2pt} + \hspace{-2pt}
        \frac{\zeta (\kappa,\hspace{-2pt} D) \eta n^2 G^2 }{2 \mu^2 } \hspace{-2pt} + \hspace{-3pt} \frac{D L_1 n \mu }{2  }, \label{eq:anchor} 
    \end{align} 
    where the last line uses Theorem \ref{thm:grad-error}, and that $ \|w_t \|^2 \le \frac{n^2 G^2 }{ \mu^2 } $ ($G \ge f (p) $ for all $p \in \D$). 
    Henceforth, we will omit the expectation operator $\E$ for simplicity. 
    
    By definition of $ (1 - \beta) \D $, we have, $ d \(x_{t+\frac{1}{2}} , x_{t+ 1} \) \le \beta $. 
    Plugging this back to (\ref{eq:anchor}), we get, in expectation, 
    \begin{align*}
        f_t ( x_t ) - f_t (x^*) 
        \le&
        \frac{1}{2 \eta } 
        \( d^2 (x_{t}, x^*) - d^2 
        (x_{t+1}, x^*) \) \nonumber \\
        &+ 
        \frac{\zeta (\kappa,\hspace{-2pt} D) \eta n^2 G^2 }{ 2 \mu^2  } + \frac{D L_1 n \mu }{2}  + d^2 ( x_{t+1}, x^* ) - d^2 ( x_{t+\frac{1}{2}}, x^* ) \\ 
        \le&  
        \frac{1}{2 \eta } 
        \( d^2 (x_{t}, x^*) - d^2 
        (x_{t+1}, x^*) \) + 
        \frac{\zeta (\kappa,\hspace{-2pt} D) \eta n^2 G^2 }{2 \mu^2  } + \frac{D L_1 n \mu }{2} + 2 \beta D. 
    \end{align*} 
    Summing over $t$ gives, in expectation, 
    \begin{align*}
        &\sum_{t=1}^T f_t ( x_t ) - f_t (x^*) \\
        \le& \frac{1}{2 \eta }
        \sum_{t=1}^T \( d^2 (x_{t}, x^*) - d^2 (x_{t+1}, x^*)  \) +
        \frac{\zeta (\kappa, D) \eta n^2 G^2 T }{2\mu^2 } + \( \frac{D L_1 n \mu }{2} + 2 \beta D \) T \\ 
        \le& 
        \frac{D^2}{2 \eta} +
        \frac{\zeta (\kappa, D) \eta n^2 G^2 T }{2\mu^2 } + \( \frac{D L_1 n \mu }{2}  + 2 \beta D \) T, 
    \end{align*}
    where the last line telescopes out the summation term. 
    
\end{proof}

\subsection{Online Convex Optimization over Hadamard Manifolds: Experiments} 

In this section, we study the effectiveness of our algorithm on the benchmark task of finding the Riemannian center of mass of several positive definite matrices \citep{bini2013computing}. The Riemannian center of mass of a set of $N$ symmetric positive definite matrices $\{A_i\}_{i=1}^N$ seeks to minimize 
\begin{align}
    f(X, \{A_i\}_{i=1}^N ) = \sum_{i=1}^N \| \log(X^{-1/2} A_i X^{-1/2}) \|_F^2,  \label{eq:exp}
\end{align} 
where $X$ is symmetric and positive definite. 

This objective is non-convex in the Euclidean sense, but is geodesically convex on the manifold of positive definite matrices. The gradient update step for objective (\ref{eq:exp}) over the manifold of symmetric and positive definite matrices is 
\begin{align*} 
    X_{t+1} = X_t^{1/2} \exp \( - \eta \;  \grad f_t \big|_{x_t} \) X_t^{1/2}. 
\end{align*}

For our experiments, we can only estimate the gradient $\grad f_t\big|_{x_t}$ by querying function values of $f_t$. We randomly sample several positive definite matrices $A_i$, and apply Algorithm \ref{alg} to search for $X$ to minimize (\ref{eq:exp}). The results are summarized in Figure \ref{fig:bandit}. 


\begin{figure}
    \centering
     \begin{subfigure}[b]{0.49\textwidth}
         \centering
         \includegraphics[width=\textwidth]{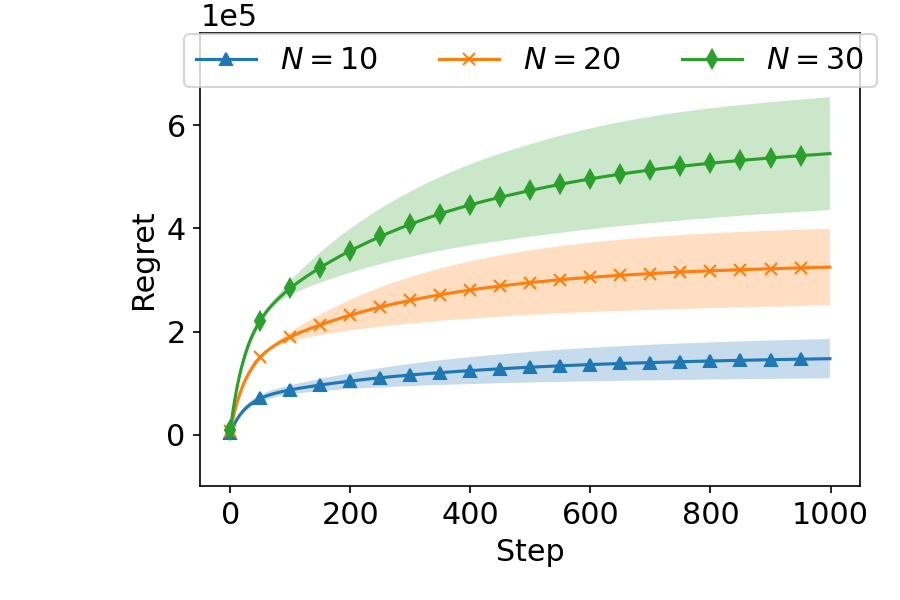}
         \caption{Dimension $ 5 \times 5$}
         \label{fig:bandit-1}
     \end{subfigure}
     \hfill
     \begin{subfigure}[b]{0.49\textwidth}
         \centering
         \includegraphics[width=\textwidth]{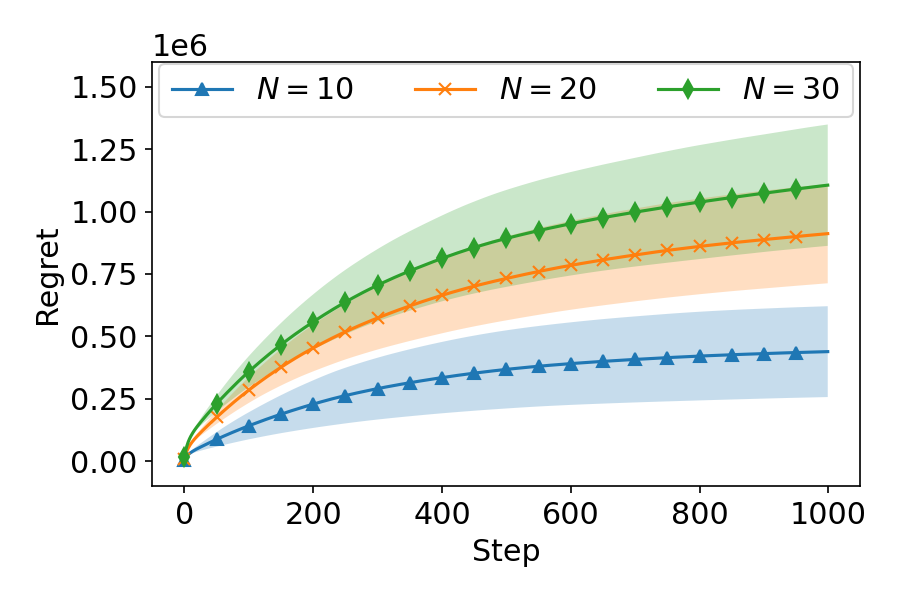}
         \caption{Dimension $10 \times 10 $}
         \label{fig:bandit-2}
     \end{subfigure}

    \caption{Results of Algorithm \ref{alg} on the task of finding the Riemannian center of mass (\ref{eq:exp}). Each line plots the average regret of 5 runs and the shaded area around the line indicates one standard deviation above and below the average. The left (resp. right) subfigure plots the result over the manifold of symmetric positive definite matrices of dimension $5 \times 5$ (resp. $10 \times 10$). The $N$ in the legends is the number of matrices used in Riemannian center of mass objective (\ref{eq:exp}).  \label{fig:bandit} }
    
\end{figure} 

\section{Conclusion}

In this paper, we study the GW convolution/approximation and gradient estimation over Riemannian manifolds. Our new decomposition trick of the GW approximation provides a tools for computing properties of the GW convolution. Using this decomposition trick, we derive a formula for how the curvature of the manifold will affect the curvature of the function via the GW convolution. 
Empowered by our decomposition trick, we introduce a new gradient estimation method over Riemannian manifolds. Over an $n$-dimensional Riemannian manifold, we improve best previous gradient estimation error \citep{li2020stochastic}. We also apply our gradient estimation method to bandit convex optimization problems, and generalize previous results from Euclidean spaces to Hadamard manifolds. 



\bibliographystyle{apalike} 
\bibliography{references} 



\end{document}